%% file: HAL.tex
\documentclass{article} 
\usepackage{iclr2027_conference,times}

\usepackage{amsmath}
\usepackage{amsthm}
\usepackage{amsfonts}
\usepackage{amssymb}

\usepackage[utf8]{inputenc}
\usepackage[T1]{fontenc}

\usepackage{float}
\usepackage{hyperref}
\usepackage{url}
\usepackage{booktabs}
\usepackage{cellspace}
\usepackage{nicefrac}
\usepackage{microtype}
\usepackage{xcolor}
\usepackage{breqn}
\usepackage{graphicx}
\usepackage{xfrac}
\usepackage{multirow}
\usepackage{pifont}
\usepackage{algorithm}
\usepackage{algpseudocode}
\usepackage{comment}

\newcommand{\crit}{\mathrm{crit}\,}

\newcommand{\timeentry}{t_{\,\mathrm{entry}}}
\newcommand{\timeentryLB}{t_{\,\mathrm{entry,LB}}}
\newcommand{\timeentryMaj}{t_{\,\mathrm{entry,maj}}}

\newcommand{\Zmaj}{Z_{\,\mathrm{maj-adv}}}

\newcommand{\Zmin}{Z_{\,\mathrm{min}}}
\newcommand{\Kmin}{\mathcal{K}_{\mathrm{min}}}
\newcommand{\Kmax}{\mathcal{K}_{\mathrm{maj}}}
\newcommand{\Qrob}{\mathcal{Q}}

\newcommand{\hinv}{h_{\mathrm{inv}}}

\newcommand{\Spec}{{\mathcal S}}
\newcommand{\supp}{\mathrm{supp}}
\newcommand{\Clip}{\mathrm{Clip}}

\newcommand{\thetaInit}{\theta_{\mathrm{init}}}
\newcommand{\ThetaInit}{\Theta_{\mathrm{init}}}

\newcommand{\jm}[1]{[[\textcolor{red}{#1}]]}

\newcommand{\gradnorm}[2]{\operatorname{grad}_{#1}#2}
\newcommand{\ngradnorm}[2]{\rho\bigl(\operatorname{grad}_{#1}#2\bigr)}

\newcommand{\jer}[1]{\textcolor{magenta}{#1}}
\newcommand{\jml}[1]{\textcolor{blue}{#1}}

\usepackage{aliascnt}
\usepackage[capitalize]{cleveref}

\newtheorem{theorem}{Theorem}[section]
\crefname{theorem}{theorem}{theorems}
\Crefname{theorem}{Theorem}{Theorems}

\newaliascnt{lemma}{theorem}
\newtheorem{lemma}[lemma]{Lemma}
\aliascntresetthe{lemma}
\crefname{lemma}{lemma}{lemmas}
\Crefname{lemma}{Lemma}{Lemmas}

\newaliascnt{corollary}{theorem}
\newtheorem{corollary}[corollary]{Corollary}
\aliascntresetthe{corollary}
\crefname{corollary}{corollary}{corollaries}
\Crefname{corollary}{Corollary}{Corollaries}

\newaliascnt{fact}{theorem}

\aliascntresetthe{fact}
\crefname{fact}{fact}{facts}
\Crefname{fact}{Fact}{Facts}

\newtheorem{definition}{Definition}[section]
\crefname{definition}{definition}{definitions}
\Crefname{definition}{Definition}{Definitions}

\newtheorem{assumption}{Assumption}[section]
\crefname{assumption}{assumption}{assumptions}
\Crefname{assumption}{Assumption}{Assumptions}

\crefname{example}{example}{examples}
\Crefname{example}{Example}{Examples}

\newtheorem{remark}{Remark}[section]
\crefname{remark}{remark}{remarks}
\Crefname{remark}{Remark}{Remarks}

\newtheorem{proposition}{Proposition}[section]
\crefname{proposition}{proposition}{propositions}
\Crefname{proposition}{Proposition}{Propositions}

\input{math_commands.tex}

\usepackage{hyperref}
\usepackage{url}

\title{On the Comparison of Optimizers for Imbalanced Learning}

\author{Clément Lezane  \\
	Université de Toulouse \\
Toulouse School of Economics \\
ANITI \\
France 
\And
François Bachoc \\
University of Lille \\
Institut Universitaire de France (IUF) \\
France
\And 
Jérôme Bolte \\
Toulouse School of Economics \\
ANITI  \\
France
\And
 Jean-Michel Loubes \\
Université de Toulouse \\
ANITI \\
 Regalia INRIA \\
 France
}

\iclrfinalcopy 
\begin{document}

\maketitle

\begin{abstract}



Data imbalance is pervasive in machine learning, from rare words and anomalies to
underrepresented patterns in heterogeneous or cross-tabulated data. We study idealized optimizers geometries in continuous time to model small-step training in deep learning. 
We assume that the source of imbalance is unobserved: the optimizer has only access to the aggregate training loss ignoring the exact contributions of the majority and minority groups. In this setting, we characterize a region where majority losses are optimized regardless of the admissible minority structure. We derive explicit equations of this zone and bounds on the time needed to leave it. These bounds exhibit a milder dependence on minority amplitude for sign, spectral, and Newton descent than for Euclidean gradient descent. Experiments with AdamW and Muon suggest similar advantages over SGD across language, tabular, and image tasks.

\end{abstract}

\section{Introduction}    
\label{section:intro}

Training data are rarely well balanced. Even when the observed class labels are
carefully supervised, the data may contain latent minority groups associated with acquisition
conditions,  rare subclasses, demographic attributes, or
rare linguistic patterns. Such groups arise for instance in 
hidden-subclass learning, long-tailed recognition
\cite{kearns2018subgroup,sohoni2020subclass,liu2019longtail}, tabular data, or subgroup fairness. Related
phenomena also occur in language modelling, where token and context
frequencies are strongly heterogeneous by nature 
\cite{kunstner2024heavytailed,kunstner2025scaling}. These groups may
be unknown before training or irreducibly imbalanced (rare words, many-features subgroups in tabular data), and in addition, in some cases methods that require their identities for 
resampling or reweighting are doomed to failure.

Building on the study of gradient descent bias in \cite{bachoc2025majorityrulesminorityloses}, we focus on training with imbalanced data and study when the optimization trajectory leaves the region of parameter space in which the majority is systematically favored.

In our training model, the user need not know the precise majority--minority decomposition \(L=L_1+L_0\), where \(L_0\) is a minority component and \(L_1\) its associated majority. Only an amplitude \(\psi>0\), measuring the impact of \(L_0\), is  known. This assumption fits many real-life scenarios, from tabular data to LLMs as presented in \Cref{section:model:imbalance}. Within this idealized training setting, our protocol  defines a minority awareness region associated with an optimizer (modelled by an ODE) and then studies the first time the resulting trajectory enters this region, leaving the place where majority is systematically favored. This entry time provides an indication of the optimizer’s ability to deal with imbalanced data. We now describe this procedure more precisely.

\paragraph{Idealizing optimizer geometry.}
A direct theoretical comparison of modern optimizers is out of reach, and
our modelling proceeds through two distinct idealizations. First, we retain
their dominant update geometries while discarding other algorithmic
features as weight decay, momentum, minibatching. Thus, classically,  gradient descent represents  SGD, sign
descent matches 
Adam-type methods, and spectral
descent corresponds to Muon and Scion-type methods \citep{pethick25a,oikonomidis2026,pethick2026}. 
Although intractable in high dimension Newton's descent is studied to understand the impact of second-order information.  Secondly, as typical learning rates are lower than $10^{-2}$, we study optimizer vector fields in
continuous time, that is solutions of $\dot{\theta}(t) \in -V_L(\theta(t))$ with $\theta(t) \in \mathbb{R}^d$ with $V_L:\mathbb{R}^d\rightrightarrows \mathbb{R}^d$ being multivalued.

\paragraph{Comparing ODEs clocks}
Continuous time
introduces a potential ambiguity in time comparisons: scaling the vector field $V_L$ scales time by the same factor, so the minority awareness time described above is not intrinsically meaningful. To solve this apparent issue, we return to discrete methods,
$\theta_{k+1} \in \theta_k-\eta V_L(\theta_k),$
for which the natural clock is \(t\simeq \eta k\). Over a range of admissible step sizes,
$$
\eta\in[\eta_{\min},\eta_{\max}],
\qquad
\kappa:=\frac{\eta_{\max}}{\eta_{\min}},
$$
the resulting clock discrepancies are therefore controlled by a factor \(\kappa\). Key to our approach is the fact that deep learning losses have rather narrow learning rates regimes across optimizers, e.g. for networks for vision tasks, learning rates typically range from
\(\eta_{\min}\simeq10^{-3}\) to \(\eta_{\max}\simeq10^{-2}\), giving
\(\kappa\simeq 20\) (see \Cref{sec-range-of-kappa} for discussions on tabular data and LLM).
Our continuous-time model is therefore justified on two grounds: as explained earlier the small-step regime makes the continuous dynamics an acceptable model, while the moderate \(\kappa\) value for DL losses ensures that the corresponding time clocks remain comparable up to a controlled factor.

\paragraph{Minority awareness and entry time.}
As in \cite{bachoc2025majorityrulesminorityloses}, we study an initial ``minority-blind'' regime in which majority learning can proceed regardless of the admissible latent minority structure. For each idealized optimizer (gradient descent, sign descent, spectral descent, or Newton's method) and each trajectory, we define \(t_{\mathrm{entry}}\) as the first entry into the minority-aware zone. Within this zone, at least one admissible decomposition yields minority improvement at the majority's expense. This need not correspond to learning the true latent subgroup: awareness concerns at least one admissible minority at the prescribed level \(\psi\). We therefore use \(t_{\mathrm{entry}}\) as an indicator of exposure to imbalance, comparing how rapidly different optimizer geometries leave the initial majority-dominated regime.


\paragraph{Main results.} Explicit equations of minority aware zones and entry-time estimates (both upper and lower) are derived for many optimizers: gradient descent,  
gradient clipping, sign descent, spectral descent, and Newton descent. Their
dependence on the minority scale \(\psi\) differs sharply: normalized
geometries can enjoy a substantially milder deterioration than Euclidean
gradient descent as imbalance becomes severe, with additional gains governed
by gradient support or matrix rank. Both are empirically related to the global dimension. Newton descent provides an idealized
second-order reference confirming that clever geometric rescaling  is key for balanced training.

\paragraph{Empirical evidences.}
We assess the relevance of entry times as indicators for imbalanced learning through two complementary approaches. First, in linear regression, where the latent decomposition and entry times can be computed, we empirically compare entry-time ratios across optimizers and find close agreement with our theoretical predictions. Second, in classification with deep networks and language modelling, where the minority level \(\psi\) and the entry time \(t_{\mathrm{entry}}\) are not directly observable, we evaluate the practical counterparts of the optimizers studied in our analysis—SGD, AdamW, and Muon—by measuring their performance on minority and majority data. Across these experiments, rescaled methods show a systematic advantage over SGD in minority learning, while the relative performance of AdamW and Muon depends on the problem geometry. In language modelling, these results are consistent with the advantage of Adam-type normalization under heavy-tailed frequencies reported in \citep{kunstner2024heavytailed,kunstner2025scaling} while corroborating the geometric benefits of Muon highlighted in recent literature \citep{vasudeva2025, wang2025muon}.

\section{A model of imbalance in supervised learning}
\label{section:model:imbalance}


\paragraph{Minorities in supervised learning}
We consider supervised learning samples
$(x_i,y_i)\in\mathcal{X}\times\mathcal{Y}$, for $1\leq i\leq n$, where $\mathcal{X}$ and $\mathcal{Y}$ are general sets.
We are interested in the situation where a small number of pairs $(x_i,y_i)$ (or equivalently of indices $i$) belong to an unknown minority group.

Typically, in binary classification where $\mathcal{Y} = \{0,1\}$, there could be a small number of $(x_i,y_i)$ with $y_i=1$, that would form a minority group (that would be easily identified). Nevertheless, we are interested in the more challenging case where this minority group cannot be identified easily.
In binary classification, we would have  
$\#\{i:y_i=0\}\simeq \#\{i:y_i=1\}$, but in the features $(x_i)_{i=1}^n$, we would have hidden minority structures. 

To model and motivate these minority structures, we 
introduce latent descriptors of the form
\[
z_i=\phi(x_i,y_i)\in\mathcal{Z},
\]
where $\phi:\mathcal{X}\times\mathcal{Y}\to\mathcal{Z}$ is an unknown map into some set $\mathcal{Z}$. Note that $\phi$ may only depend on  the features, in which case it would hold
$\phi(x_i,y_i)=\phi(x_i)$.
Fixing a set $\mathcal{A}\subset\mathcal{Z}$, the index set
\begin{equation}
\label{eq:IndAset}
I_{\mathcal{A}}:=\{i\in\{1,\dots,n\}:\phi(x_i,y_i)\in\mathcal{A}\}
\end{equation}
may yield a latent minority group $(x_i,y_i)_{i \in I_{\mathcal{A}}}$.
Let us give concrete cases.

\textbf{a) Tabular Data.} Suppose $z_{\phi} \in \mathcal{X} = \mathbb{R}^d  $, define
$\phi(x,y)= z_{\phi}^{\top}x \in \R$ and let $\phi(x,y)$ be a score such as age, weight, or wealth.
Then for each threshold $t\in\mathbb{R}$ one may take
$\mathcal{A}_t=(-\infty,t]$. In that case \(
    I_{\mathcal{A}_t}=\{i:\phi(x_i)\le t\}
\) is a minority set of individuals with particularly low scores.
   
\textbf{b) Hidden subclasses in classification.} Consider multi-class classification where the labels \(y_i \in \{1,\dots,C\}\) are balanced. Nevertheless, it is still possible that each sample also belongs to a subclass \(h_i \in \{1,\dots,M\}\), corresponding to a viewpoint, shooting condition, scene configuration, scanner type, or image resolution. Fix a rare regime \(h_\star \in \{1,\dots,M\}\), and define \(\phi(x,y) \in \{1,\dots,M\}\) with \(\mathcal{A}=\{h_\star\}\). Then
\(
 I_{\mathcal{A}}=\{i : \phi(x_i,y_i) = h_\star\}.
\)
The minority group is then the rare visual regime \(h_\star\). This corresponds to the hidden-subclass setting of \cite{sohoni2020subclass}: the imbalance does not lie in the labels themselves, but within each class; see also \cite{kearns2018subgroup}. 
 

\textbf{c) Rare-word counts in autoregressive language models.} Let
$(s_t)_{1\leq t\leq N}$ be a tokenized corpus with vocabulary
$\mathcal{V}$. Each training sample is $ x_t=(s_1,\dots,s_{t-1}), y_t=s_t.$ For each token $v\in\mathcal{V}$, let
$N_v:=\#\{t:y_t=v\}$ and assume sorted counts $N_{(1)} \geq \cdots \geq N_{(|\mathcal{V}|)} $. We set $\phi(x_t,y_t)=N_{y_t}$, and for a given rank $r$, $\mathcal{A}_r=\{1,\ldots,N_{(r)}\}$ isolates a minority
\[ I_{\mathcal{A}_r}=\{t:N_{y_t}\le N_{(r)}\}\]
whose target tokens fall into the long tail of the distribution \citep{piantadosi2014zipf, kunstner2024heavytailed,kunstner2025scaling, radford2019language}. If the empirical frequencies exhibit a power-law decay, $N_{(r)} \propto r^{-\alpha} $, then
\[
\frac{|\mathcal{I}_{ \mathcal{A}_r} |}{N} \approx \frac{\sum_{k=r}^{|\mathcal{V}|} k^{-\alpha}}{\sum_{k=1}^{|\mathcal{V}|} k^{-\alpha}}.
\]

\paragraph{Empirical risk minimization}

Let $\ell(\theta;x,y) \in \mathbb{R}$ be a loss (typically associated to a neural network with parameters $\theta \in\R^d$). Define the empirical risk  $ L(\theta)=\frac{1}{n}\sum_{i=1}^n \ell(\theta;x_i,y_i). $
Given a minority group $I_{\mathcal{A}}$ as above, we can define the minority loss and the corresponding majority loss
\begin{equation} \label{eq:Lzero}
L_0(\theta)=\frac{1}{n}\sum_{i \in I_{\mathcal{A}}} \ell(\theta;x_i,y_i), \quad L_1(\theta)=\frac{1}{n}\sum_{i \in \{1,\dots,n\} \backslash I_{\mathcal{A}}} \ell(\theta;x_i,y_i).
\end{equation}
We have the loss decomposition $L = L_0 + L_1$, and, as discussed above, this decomposition is unknown in general, and even more, multiple notions of minority simultaneously give rise to multiple decompositions.
This structural ambiguity is inherent to many learning problems; it appears in subgroup fairness \citep{kearns2018subgroup}, hidden-subclass problems \citep{oakden2020hidden, sohoni2020subclass}, and
long-tailed recognition \citep{kang2020decoupling, liu2019longtail}.
In the sequel, we adopt a general yet theoretically tractable viewpoint, considering a general loss function $L$ and a general set of admissible minority losses $L_0$.

\paragraph{The learning loss model}

With the above motivation in mind, for $d \in \mathbb{N}^*$, consider a
loss $L: \mathbb{R}^d \to \mathbb{R}$, for which we make the standing assumption
\begin{assumption}[Loss assumptions]
\label{ass-cost}
$L$ is $C^2$ with
$L_\star := \inf_{\theta \in \mathbb{R}^d} L(\theta) > -\infty$.
\end{assumption}

Motivated by the above, we assume that $L$ is subject to an unknown
decomposition
\[
\forall \theta \in \mathbb{R}^d, \quad L(\theta) = L_1(\theta) + L_0(\theta),
\]
where $L_1,L_0: \mathbb{R}^d \to \mathbb{R}$ are twice continuously differentiable
functions and, for some $\psi,D>0$,
\begin{equation}
\label{def-K}
L_0\in \Kmin
:=
\left\{
M \in \mathcal{C}^2(\mathbb{R}^d, \mathbb{R}) \;:\;
\forall \theta \in \mathbb{R}^d,\;
\|\nabla M(\theta)\|_2 \leq \psi\big(\|\theta\|_2 + D\big)
\right\}.
\end{equation}

The abstract decomposition $L = L_1 + L_0$ should therefore be interpreted as a way to model uncertainty over minority structures while keeping the total objective fixed. 

In view of the above, $\psi$ may be understood as a bound on the
effective mass of the latent minority group. Coming back to Equations~\ref{eq:IndAset} and \ref{eq:Lzero}, it is informally interpreted as
\[
\psi \approx 
\frac{n_0}{n_0+n_1} G,
\]
where $n_0$ is the cardinality of $\mathcal{I}_{\mathcal{A}}$, $n_1 = n - n_0$ and $G$ is the ``average'' Lipschitz constant of the gradients of $\ell(\cdot;x_i,y_i)$, for $i \in \mathcal{I}_{\mathcal{A}}$.
When the loss function has bounded Hessian, we consider $\psi := \sigma_{\max}(\nabla^2 L_0)$  (largest singular value). Hence, as further discussed and quantified in \cite{bachoc2025majorityrulesminorityloses}, $L_0$ can be small because the  minority losses $\ell(\cdot;x_i,y_i)$ are scarce or because they have a small variability.


\section{Learning process and minority aware zone}
\label{section:learning:process:majority}

\subsection{Learning with descent methods}
\label{subsection:learning:with:descent}

The learning process is modeled by a continuous-time dynamical system driven by
an optimizer $V_L:\mathbb{R}^d\rightrightarrows \mathbb{R}^d,$ 
which may be multivalued, that is $V_L(\theta) \subset \mathbb{R}^d$ for $\theta \in \mathbb{R}^d$. 
We refer to \citet{aubin1984differential} for the theory.
In the remainder we  assume the following.
\begin{assumption}[Dynamics assumptions]
\label{ass-VL}
$V_L$ is nonempty convex compact valued, has closed graph and is locally bounded.
\end{assumption} 
From \Cref{Th:aubincell},
for any $\theta_{\text{init}} \in \mathbb{R}^d$,
there exists an absolutely continuous curve $\theta$ and $T>0$ satisfying
\begin{equation}
\label{eq-dynamical-system}
\dot{\theta}(t) \in -V_L(\theta(t)) \quad \text{ for a.e. } t\in [0,T), \qquad \theta(0) = \theta_{\text{init}}.
\end{equation}
In the various settings we adopt, we will see that $T=+\infty$. 

Denote by $dL(\theta)$ the differential mapping of $L$ at $\theta$.
In the sequel, we shall consider various choices for $V_L$, for which the 
key mechanism of descent property holds:
\begin{eqnarray}
\label{eqdescent}
&&dL(\theta)=0\Leftrightarrow V_L(\theta)\ni 0 \\
&&dL(\theta)(v)=\langle v,\nabla L(\theta)\rangle> 0, 
 \: \forall v\in V_L(\theta), \: \forall \theta\in\mathbb{R}^d\setminus dL^{-1}(\{0\}),\label{eqdescent1}
\end{eqnarray}
which ensures that the loss  decreases along trajectories of
\Eqref{eq-dynamical-system}.

Five idealized descent mechanisms are considered: gradient descent,
gradient clipping, sign descent, spectral descent, and Newton descent.
They share a common variational construction: gradient directions
are selected by maximizing the differential over a norm ball,
and the speed may then be adjusted through clipping or normalization.
For simplicity, we first present this construction for a fixed norm
\(\|\cdot\|\) on \(\mathbb{R}^d\). Parameter-dependent
norms \(\|\cdot\|_\theta\) would accommodate Newton descent.

\paragraph{Normal gradients} For a linear form \(\ell\) on \(\mathbb{R}^d\), define its dual norm by 
$ \|\ell\|_{\ast}
:=
\max_{\|z\|\leq1}\ell[z].$
For a differentiable function \(L\), the corresponding norm-gradient writes
\begin{equation}
\label{eq:scaled-steepest-template}
\gradnorm{\|\cdot\|}{L(\theta)}
:=
\argmax_{\|z\|\leq \|\mathrm{d}L(\theta)\|_{\ast}}
\mathrm{d}L(\theta)[z],
\end{equation}
while the unit norm-gradient directions are
$\displaystyle
\ngradnorm{\|\cdot\|}{L(\theta)}
:=
\argmax_{\|z\|\leq1}\mathrm{d}L(\theta)[z].
$
The gradient $\operatorname{grad}_{\|\cdot\|}$ could be multivalued depending on the curvature properties of the unit ball of the norm $\|\cdot\|$. From the definitions, we infer:
\begin{equation}\label{eq:l_op}
   \forall v \in  \gradnorm{\|\cdot\|}{L(\theta)}, \,\,\,\, \|v\|=\|\mathrm{d}L(\theta)\|_{\ast} \,\,\,\, \mbox{ with  } \,\,\,\,\mathrm{d}L(\theta)[v]=\|\mathrm{d}L(\theta)\|_{\ast}^2.
\end{equation}

 By standard arguments the multivalued mappings $\operatorname{grad}_{\|\cdot\|}$ and $\ngradnorm{\|\cdot\|}{L(\theta)}$ are closed maps with nonempty convex compact values. 

 In the remainder the normal gradient descent and the normalized gradient descent will play a key role: 
\begin{equation*}
    \dot \theta(t)\in -\gradnorm{\|\cdot\|}{L(\theta(t))} \mbox{ and }\dot \theta(t)\in -\ngradnorm{\|\cdot\|}{L(\theta(t))}.
\end{equation*}
We have the following well-known basic properties.

\paragraph{Gradient descent:}
For Euclidean gradient descent, choose \(\|\cdot\|=\|\cdot\|_2\) so that \Eqref{eq:scaled-steepest-template} gives
\[
V_L(\theta)=\gradnorm{\|\cdot\|_2}{L(\theta)}=\nabla L(\theta)
,
\qquad
\dot{\theta}(t)=-\nabla L(\theta(t)).
\]

\paragraph{Gradient clipping:}
Fix \(\tau_{\rm clip}>0\), and define
\(\displaystyle 
\Clip_{\tau_{\rm clip}}(g)
:=
\begin{cases}
g, & \|g\|_2\leq \tau_{\rm clip},\\
\tau_{\rm clip}\dfrac{g}{\|g\|_2}, & \|g\|_2>\tau_{\rm clip},
\end{cases}
\)
whose dynamics is $\dot\theta(t)=-\Clip_{\tau_{\rm clip}}(\nabla L(\theta(t))). $

\paragraph{Sign descent:}
Choose \(\|\cdot\|=\|\cdot\|_\infty\), and consider $V_L=\ngradnorm{\|\cdot\|_{\infty}}{L(\theta)}$. Define the scalar
multivalued sign by
$ \sign(s)=1$ for $s>0$, $ \sign(s)=-1$ for $s<0$, and $\sign(0)=[-1,1]$
 and extend it component-wise. 
 The dynamics given by $V_L$ is called {\em sign descent}:
\[
\dot\theta(t)\in-\sign(\nabla L(\theta(t))).
\]

\paragraph{Spectral descent:}
Assume the parameter space is of the form $\mathbb{R}^{a\times b}$. The differential is
represented by the matrix gradient through
\[
\mathrm{d}L(\Theta)[H]=\operatorname{tr}\bigl(\nabla L(\Theta)^\top H\bigr)
\qquad \forall \Theta \in\mathbb{R}^{a\times b}
\qquad \forall H\in\mathbb{R}^{a\times b}.
\] 
The  matrix Spectral norm (w.r.t. to Euclidean norm) is denoted by $\|\cdot\|_{\rm sp}$, its dual norm --the nuclear norm-- is denoted by 
$\|\cdot\|_{\mathrm{nuc}}$. Spectral descent is the corresponding normalized spectral-norm descent:
\begin{equation} \label{eq:normalized:spec:desc}
\dot\Theta(t)\in
-\ngradnorm{\|\cdot\|_{\mathrm{sp}}}{ L(\Theta(t))}.
\end{equation}
More details about spectral descent are provided in \Cref{section:appendix-spectral-Descent}

\paragraph{Newton descent:} 
Assume there exists \(\mu>0\) such that
\begin{equation}\label{ass-newton}
\langle u,\nabla^2L(\theta)u\rangle\geq \mu\|u\|_2^2
\qquad
\forall \theta,u\in\mathbb{R}^d
\end{equation}
and define the {\em Hessian norm} at \(\theta\) in $\R^d$ by
\(
\|u\|_\theta^2:=\langle u,\nabla^2L(\theta)u\rangle, \mbox{ for any $u$ in $\R^d$.}
\)
\Eqref{eq:scaled-steepest-template} yields the Newton field
$
V_L(\theta)
=
\nabla^2L(\theta)^{-1}\nabla L(\theta)$ with $
\dot{\theta}(t)
=
-\nabla^2L(\theta(t))^{-1}\nabla L(\theta(t)).
$

\subsection{Majority adverse zone and minority aware zone}


For a fixed loss $L$, we use the admissible minority class \(\Kmin\) defined in
\Cref{def-K}. The corresponding admissible ``majority losses"\footnote{The term majority refers to the potential magnitude of the losses in $\Kmax$ not necessarily to their actual magnitude as they may be for instance null functions (in the case where $L \in \Kmin$ and $L_0 = L$).} are
\begin{equation}
\label{eq:Kmax}
\Kmax(L):=\{L-M:\ M\in\Kmin\}.
\end{equation}

\medskip

For the sequel, fix a decomposition $L=L_1+L_0$ with $L_0\in\Kmin$ and
$L_1=L-L_0\in\Kmax(L)$, and consider a possibly multivalued optimizer
$V_L:\mathbb{R}^d\rightrightarrows\mathbb{R}^d$.  Based on the majority adverse zone introduced in
\cite{bachoc2025majorityrulesminorityloses}, we introduce the following
variation:

\begin{definition}[Majority adverse zone]
\label{def-Maj}
Consider a loss function $L\in \mathcal{C}^2(\mathbb{R}^d,\mathbb{R})$, an
optimizer $V_L:\mathbb{R}^d\rightrightarrows\mathbb{R}^d$ satisfying \Cref{ass-VL}, and
$L_0\in\Kmin$, yielding a fixed decomposition $L=L_1+L_0$. The majority adverse zone associated with
$(L,V_L,L_0)$ is
\begin{equation}
\label{eq:Zmaj-Q}
\Zmaj(L_0)
:=
\left\{\theta\in\mathbb{R}^d:
\sup_{v\in V_L(\theta)}
\langle v,\nabla L_1(\theta)\rangle<0
\right\}.
\end{equation}
\end{definition}

Equivalently, $\theta\in\Zmaj(L_0)$ iff
$
\langle v,\nabla L_1(\theta)\rangle<0$ for all $v\in V_L(\theta)
$, which matches the definition in
\cite{bachoc2025majorityrulesminorityloses} when $V_L=\nabla L$ and with $\leq$ in \Eqref{eq:Zmaj-Q}, see \Cref{rem:strict:ineq}.

As the majority-minority  decomposition is unknown,
we take the union of the majority adverse zones over all admissible
minority components, to obtain what we called the minority aware zone. 


\begin{definition}[Minority aware zone]
\label{def-RMaj}
Consider a loss function $L\in \mathcal{C}^2(\mathbb{R}^d,\mathbb{R})$ and an
optimizer $V_L:\mathbb{R}^d\rightrightarrows\mathbb{R}^d$satisfying \Cref{ass-VL}. The minority aware zone is
\begin{equation*}
\Zmin
:=
\bigcup_{L_0\in\Kmin}\Zmaj(L_0).
\end{equation*}
\end{definition}

 As explained in the next lemma,  the complement of $\Zmin$ is, roughly speaking, a
minority-blind regime: the learning dynamic remains
majority-favouring whatever the decomposition. Alternatively  the minority aware zone may be seen as the locus of points where the optimizer favors at least one  minority in some decomposition. In $\Zmin$ complete blindness ceases; see \Cref{lem:minority-activation}.

\begin{lemma}[Uniform majority dominance]
\label{thm:uniform-majority-dominance}
Under Assumptions~2.1 and~3.1, suppose that $V_L$ is
single-valued. Then, for every $\theta\in\mathbb{R}^d$,
\[
\theta\notin\Zmin
\quad\Longleftrightarrow\quad
dL_1(\theta)[-V_L(\theta)]\leq 0
\quad\text{for every }L_1\in\Kmax(L).
\]
Thus, every admissible majority loss is nonincreasing
along a trajectory that remains outside $\Zmin$.
\end{lemma}

See also \Cref{thm:uniform-majority-dominance-general} and the Appendix for other considerations on the minority awareness zone. An illustration is provided in Figure \ref{fig:robust-zones}.



\begin{figure}
    \centering
    \includegraphics[scale=0.22]{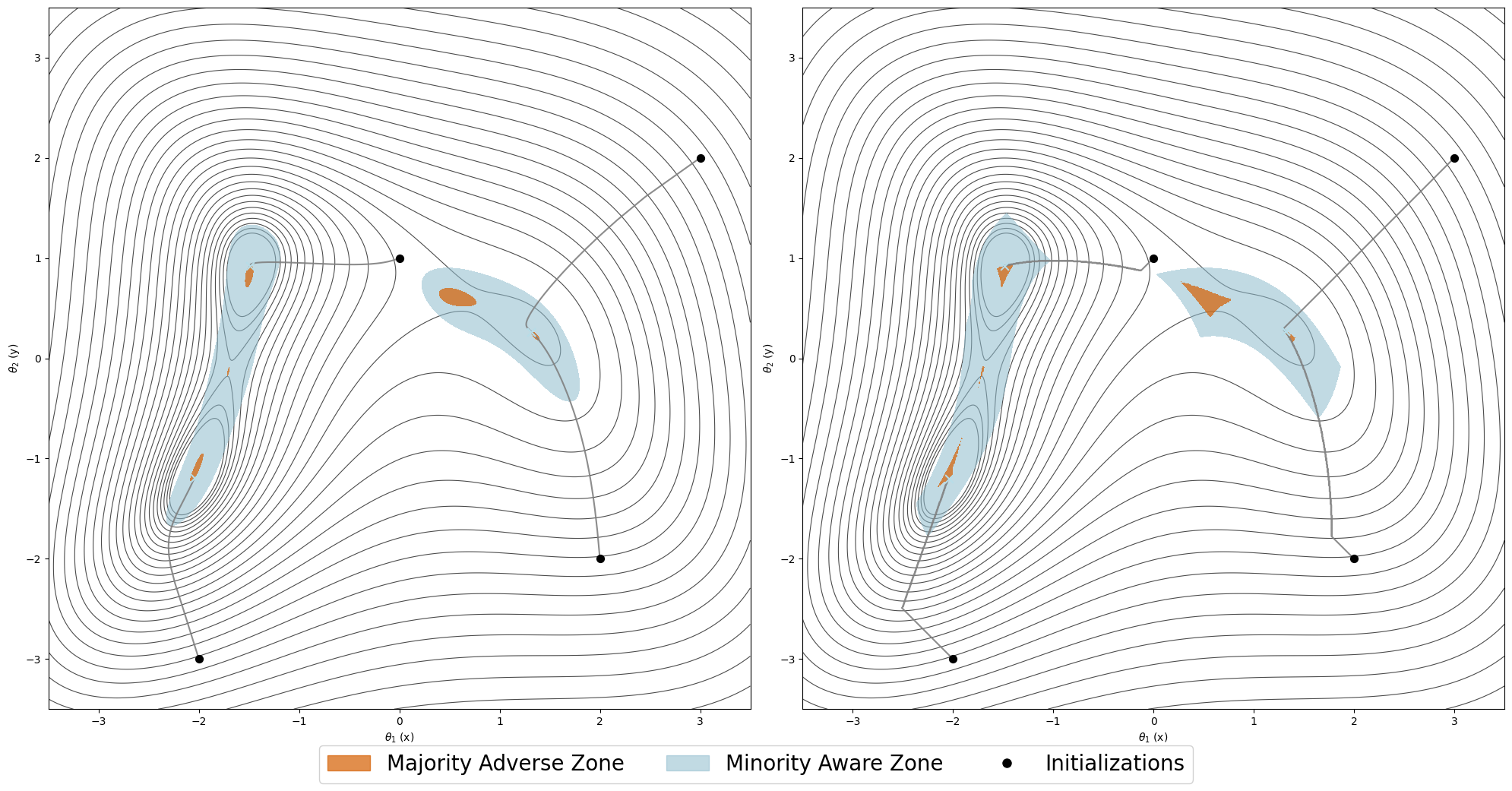}
    \caption{Landscape $L$, majority adverse, minority aware zones, learning curves for gradient descent (left) and sign descent (right). 
    }
       \label{fig:robust-zones}
\end{figure}

\subsection{Variational characterization of the minority aware zone}

\paragraph{Variational representations}
Minority aware zones can be represented as level sets of functions that, in
all cases considered below, admit explicit formulas.
For $L_0\in\Kmin$, set $L_1=L-L_0$ and define
\begin{equation*}
Q_{L_0}(\theta)
:=
\sup_{v\in V_L(\theta)}
\langle v,\nabla L_1(\theta)\rangle .
\end{equation*}

This function measures the worst, over all admissible velocities,
instantaneous variation of the majority loss. It is made to detect zones where $L_1$ does not decrease, indeed: 
$\Zmaj(L_0)=\{\theta\in\mathbb{R}^d:\ Q_{L_0}(\theta)<0\}$. We now aggregate over all
admissible minority losses by setting
\begin{equation}
\label{eq:calQ-multivalued}
\Qrob(\theta)
:=
\inf_{L_0\in\Kmin}Q_{L_0}(\theta)
=
\inf_{L_0\in\Kmin}
\sup_{v\in V_L(\theta)}
\left\langle v,\nabla (L-L_0)(\theta)\right\rangle,
\end{equation}
which is finite-valued due to the definition of $\Kmin$. We immediately prove the following lemma.
\begin{lemma}[Minority zone and \texorpdfstring{$\Qrob$}{Q}]
\label{lemma:robust-zone-calQ-main} Under \Cref{ass-VL}, and by definition of $\Kmin$, we have:
\[
\theta\in\Zmin
\quad\Longleftrightarrow\quad
\Qrob(\theta)<0, \quad\mbox{ for every $\theta\in\mathbb{R}^d$}.
\]
\end{lemma}


\paragraph{Equations for the minority aware zone}
In order to obtain explicit equations defining $\mathcal Q$, we shall use the following key results:
\begin{lemma}[Pointwise gradient envelope for \texorpdfstring{$\Kmin$}{Kmin}]
\label{lemma:pointwise-envelope-K}
Under \Cref{ass-cost}, one has
\begin{equation}
\label{eq:pointwise-envelope-K}
\{\nabla L_0(\theta):L_0\in\Kmin\}
=
\{\xi\in\mathbb{R}^d:\|\xi\|_2\leq \psi(\|\theta\|_2+D)\}\,\mbox{ for
every $\theta\in\mathbb{R}^d$.}\end{equation}
\end{lemma}

\begin{theorem}[Multivalued minority aware zone]
\label{thm-calQ}
Under Assumptions \ref{ass-cost} and \ref{ass-VL}, one has, for every
$\theta\in\mathbb{R}^d$,
\begin{equation}
\label{eq:calQ-gradient-multivalued}
\Qrob(\theta) = \sup_{v\in V_L(\theta)}  \Big( \langle v, \nabla L(\theta) \rangle- \psi(\|\theta\|_2+D) \| v\|_2 \Big).
\end{equation}
Consequently, $\theta\in\Zmin
\quad\Longleftrightarrow\quad
\sup_{v\in V_L(\theta)}  \Big( \langle v, \nabla L(\theta) \rangle- \psi(\|\theta\|_2+D) \| v\|_2 \Big) < 0 .$
\end{theorem}
\Cref{lemma:pointwise-envelope-K} and \Cref{thm-calQ} are proved in \Cref{sec-appendix-proof-section-3}. Next, we define two core quantities. For a norm \(\|\cdot\|\) on \(\mathbb{R}^d\), set
\begin{equation}
\label{mtheta-cnorm}
m_{\|\cdot\|}(\theta)
:=
\inf_{v\in \ngradnorm{\|\cdot\|}{L(\theta)}}\|v\|_2, \quad c_{\| \cdot \|} := \inf_{v \in \mathbb{R}^d \setminus\{0\}} \frac{\|v\|_2}{\| v\|}. 
\end{equation}

Using \Cref{thm-calQ} , we may now derive the explicit expressions of the minority aware zones.  Results are gathered  in \Cref{table:zoology:minority:zone} and details can be found in \Cref{cor-robust-zone-norm-gradient-fields} and \Cref{sec-appendix-proof-section-3}.
Here $\crit$ denotes the set of critical points, $|\supp(\cdot)| = \| \cdot \|_0$ and $\operatorname{rank}(\cdot)$ is the rank of a matrix.

\begin{table}[h]
    \centering
    \begin{tabular}{ll}
        \toprule
         \textbf{Dynamics} &  \textbf{Minority zones} $\Zmin $ \hfill $\theta\notin \crit L$\\
         \midrule 
        $V_L(\theta)=\ngradnorm{\|\cdot\|}{L(\theta)}$& $\|\mathrm{d}L(\theta)\|_{\ast}< \psi m_{\|\cdot\|}(\theta)(\|\theta\|_2+D).$ \\
        $V_L(\theta)=\gradnorm{\|\cdot\|}{L(\theta)}$ & $\|\mathrm{d}L(\theta)\|_{\ast}^2< \psi m_{\|\cdot\|}(\theta) \|\mathrm{d}L(\theta)\|_{\ast}(\|\theta\|_2+D)$ \\
        \midrule
         Gradient descent & $\|\nabla L(\theta)\|_{2}< \psi (\|\theta\|_2+D)$ \\
        Gradient clipping & $\|\nabla L(\theta)\|_{2}< \psi (\|\theta\|_2+D)$\\
        Sign descent & $\|\nabla L(\theta)\|_1 < \psi\sqrt{|\supp(\nabla L(\theta))|}(\|\theta\|_2+D)$\\
        Spectral descent & $\|\nabla L(\Theta)\|_{\mathrm{nuc}}<\psi\sqrt{\operatorname{rank}(\nabla L(\Theta))}(\|\Theta\|_F+D) $\\
        Newton descent & $\left\langle \nabla^2L(\theta)^{-1}\nabla L(\theta),\nabla L(\theta)\right\rangle < \psi\|\nabla^2L(\theta)^{-1}\nabla L(\theta) \|_2(\|\theta\|_2+D)$ (\textbf{$\ast$})\\
        \bottomrule
    \end{tabular}
    \caption{\footnotesize Zoology of minority aware zones for different dynamics. ($\ast$) Strongly convex loss.}
    \label{table:zoology:minority:zone}
\end{table}

\section{Entry time and group minority awareness}
\label{section:training:times}


\begin{definition}[Entry time]
\label{def-entry-time}
Under \Cref{ass-cost,ass-VL}, let $\theta$ be an absolutely continuous solution  to 
\[
\theta(0)=\thetaInit,
\qquad
\dot\theta(t)\in -V_L(\theta(t))
\quad\text{for a.e. }t\geq0.
\]
The entry time along the trajectory \(\theta\) is defined as
\(
\timeentry(\theta)
:=
\inf\{t\geq0:\theta(t)\in\Zmin\} \in [0,+\infty]
\). 
When $\theta$ is unique, we may write $\timeentry(\theta)=\timeentry(\thetaInit)$.
\end{definition}

Entry time is a critical marker that structures the learning process in phases.  Before $\timeentry$, there is a {\em majority dominance} phase: majority traits are learned independently of the minority structure, see \Cref{thm:uniform-majority-dominance} and \Cref{lemma-majority-increase}. Then when the trajectory enters $\Zmin$, some minority structures start being  visible,  see \Cref{lem:minority-activation}. It is quite easy to see that the finiteness of $\timeentry(\theta)$ does not mean that the trajectory stays indefinitely in the minority aware zone, it may go out, but then we can provide similar estimates with new initializations. Entry times serve as theoretical indicators  for the onset of majority-adverse behavior as described in \cite{bachoc2025majorityrulesminorityloses}.

\paragraph{Theoretical bounds for entry time}
Recall \(m_{\|\cdot\|}\) in \Eqref{mtheta-cnorm}, and for a trajectory \(\theta\), consider the {\em $m_{\|\cdot\|}$-average along $\theta$}:
\begin{align}\label{eq:maverage}
&\bar m(t):=\frac1t\int_0^t m_{\|\cdot\|}(\theta(s))\,ds
\quad(t>0),\\
\label{eq:barMtheta}
&\bar M_\theta:=
\frac1{\timeentry(\theta)}
\int_0^{\timeentry(\theta)}
m_{\|\cdot\|}(\theta(s))\,ds  \mbox{ whenever \(0<\timeentry(\theta)<\infty\).}
\end{align}

\begin{assumption}[Quadratic growth at infinity]
\label{assumption-l-smooth} 
Let the constant $D$ be the same as in the definition of $\Kmin$ (see \Eqref{def-K}). Assume that there exists a positive constant $\lambda_{\max} >0$ such that $\forall \theta \in \mathbb{R}^d, \quad L(\theta) - L_\star  \leq \lambda_{\max}( \| \theta\|_2 + D)^2.$ 
\end{assumption}

\begin{theorem}[Entry time for normalized norm-gradient]
\label{thm-ngrad-entry-time-upper-1}
Under \Cref{ass-cost,ass-VL}, let \(\theta\) solve
\[
\dot\theta(t)\in
-\ngradnorm{\|\cdot\|}{L(\theta(t))},
\qquad
\theta(0)=\thetaInit.
\]
If \(\nabla L(\thetaInit)\neq 0\), then
$ \timeentry(\theta)
\leq
\frac{L(\thetaInit)-L_\star}
{\psi D\bar M_\theta}.
$
If \Cref{assumption-l-smooth} also holds and
\(\timeentry(\theta)>0\), then
$ \timeentry(\theta)
\leq
\frac{2\sqrt{\lambda_{\max}}}
{\psi\bar M_\theta}
\sqrt{L(\thetaInit)-L_\star}.$
\end{theorem}

\begin{theorem}[Entry time for norm-gradient]
\label{thm-gradnorm-entry-time-upper-1}
Under \Cref{ass-cost,ass-VL}, let \(\theta\) solve
\[
\dot\theta(t)\in-\gradnorm{\|\cdot\|}{L(\theta(t))},
\qquad
\theta(0)=\thetaInit.
\]
If $\nabla L(\thetaInit)\neq 0$, one has 
$ \timeentry(\theta)
\leq
\frac{L(\thetaInit)-L_\star}{\psi^2 D^2 \bar M_\theta^2 }.
$
If \Cref{assumption-l-smooth} also holds, and \(\timeentry(\theta)>0\), then
\[
\timeentry(\theta)
\leq
\frac{\lambda_{\max}}{\psi^2\bar M_\theta^2}
\begin{cases} 
\frac{L(\thetaInit)-L_\star}{\lambda_{\max}D^2}
& ~ ~ \text{if}  ~ ~
\frac{L(\thetaInit)-L_\star}{\lambda_{\max}D^2} \le e
\\
e
\log \left(  \frac{L(\thetaInit)-L_\star}{\lambda_{\max}D^2} \right)
& ~ ~ \text{if}  ~ ~
\frac{L(\thetaInit)-L_\star}{\lambda_{\max}D^2} > e
\end{cases}.
\]
\end{theorem}

By applying \Cref{thm-ngrad-entry-time-upper-1,thm-gradnorm-entry-time-upper-1} for the specific optimization dynamics introduced in \Cref{subsection:learning:with:descent}, we derive various upper bounds of entry time presented in \Cref{table-entry-time} (see details in Sections \ref{Proof-entry-time-1}, \ref{sec:appendix-entry-time-clipped}, \ref{sec:appendix-entry-time-newton} and \ref{sec:appendix-entry-time-applications}). 
\begin{table}[H]
    \centering
    \begin{tabular}{lll}
        \toprule
         \textbf{Dynamics} &  \textbf{Entry times in general}  &  \textbf{Entry times with}  \Cref{assumption-l-smooth}  \\
        \midrule
         Gradient descent & $\timeentry^{\rm GD} = \mathcal{O}  \left(  \frac{L(\thetaInit)-L_\star}{\psi^2 D^2} \right) (\star) $ & 
$\timeentry^{\rm GD} = \mathcal{O} \left(\frac{\lambda_{\max}}{\psi^2}
\log (  \frac{L(\thetaInit)-L_\star}{\lambda_{\max}D^2} ) \right)$  \\
        Gradient clipping & $\timeentry^{\rm GC} = \mathcal{O} \left( \frac{L(\thetaInit)-L_\star }{\psi D \min \{\tau_{\rm clip},\psi D \} }\right) $  & 
$\timeentry^{\rm GC}
= \mathcal{O} 
\left( \frac{\sqrt{\lambda_{\max} (L(\thetaInit)-L_\star)}}{\psi \min \{\tau_{\rm clip},\psi D \}} \right)
$\\
        Sign descent & $\timeentry^{\rm SD} = \mathcal{O}  \left(
\frac{L(\thetaInit)-L_\star }{\psi D \bar M_{\rm SD}}
\right) (\star)$  & 
$\timeentry^{\rm SD}
= \mathcal{O} \left( 
\frac{\sqrt{\lambda_{\max} (L(\thetaInit)-L_\star)}}{\psi\bar M_{\rm SD}} \right)$\\
        Spectral descent & $\timeentry^{\rm Spec} = \mathcal{O}   \left(
\frac{L(\Theta_{\rm init})-L_\star}{\psi D \bar R} \right) (\star)$  & 
$\timeentry^{\rm Spec}
= \mathcal{O} \left( 
\frac{\sqrt{\lambda_{\max} (L(\Theta_{\rm init})-L_\star)}}{\psi\bar R} \right) $ \\
        Newton descent & $ \timeentry^{\rm Newton} = \mathcal{O}\left( \log ( \frac{\| \nabla L(\thetaInit) \|_2 }{\psi D } )\right)$  & 
$ \timeentry^{\rm Newton} = \mathcal{O} \left( \log ( \frac{\| \nabla L(\thetaInit) \|_2 }{\psi D } )\right)$\\
        \bottomrule
    \end{tabular}
    \caption{\footnotesize Zoology of entry times for different dynamics (see also \Cref{table-reference-entry-time} in \Cref{sec-appendix-c-3}). 
    The term $\bar M_{\rm SD}$ denotes the time-average of square root of the support size of $\nabla L$ prior to the entry time, and  $\bar R$ denotes the time-average of square root of rank of $\nabla L$ over the same period. For the entry times in $ (\star)$, we also provide lower bounds, that is, examples of $L$ matching the upper bounds in \Cref{sec-lower-bounds}.}
    \label{table-entry-time}
\end{table}
\begin{remark}
The relevance of entry times regarding the optimization of the actual geometries is also discussed in a particular case in \Cref{derdeder}.
\end{remark}

Within our continuous-time idealization of the small-step regime,
\Cref{table-entry-time} highlights two ways in which optimizer
geometry can facilitate departure from majority dominance:\\
--- The entry-time bounds
for sign and spectral descent scale as $\mathcal{O}(\psi^{-1})$,
rather than $\mathcal{O}(\psi^{-2})$ for gradient descent.
This suggests that normalization can mitigate the delay associated
with increasingly weak minority components. Newton descent exhibits very weak 
dependence, of order $\log(1/\psi)$. 
In conclusion, the more severe the imbalance, the larger the advantage of normalized geometries of the optimizers.
\\
--- The factors $\bar M_{\rm SD}$ and $\bar R$ may further tighten
the sign and spectral descent bounds when trajectories maintain
large gradient support or rank before entry.
We note that comparing upper bounds does not establish that sign or spectral descent always enters earlier than GD on the same problem.  

As a conclusion, when imbalance is suspected but group labels are unavailable ---hidden subclasses, latent data, rare patterns--- reweighting and resampling cannot be applied directly. Our results identify optimizer choice as a practical response: normalized and spectral geometries yield minority-awareness time bounds with a milder dependence on minority amplitude than Euclidean gradient descent. Together with the experiments in \Cref{section:numerical}, these findings support favoring rescaled methods as for instance AdamW or Muon for training under imbalance without explicit knowledge of minority groups.

\section{Numerical Experiments}
\label{section:numerical}

We first provide a sanity check of our $\timeentry$ bounds on quadratic problems. We then test whether normalized gradient methods yield more balanced outcomes  as our theory suggests. In practical deep learning settings, theoretical quantities like  $\psi$ and $\timeentry$ are intractable,  we thus measure empirical accuracy and loss to assess performances with respect to imbalance.\\
\textbf{Entry times through linear regression.}
In this experiment, we seek empirical scaling relations linking
entry times to minority amplitude $\psi$ and dimension $d$.
Using the Superconductor and YearPredictionMSD datasets, we apply
principal component analysis (PCA) to construct a coherent family
of linear-regression problems across dimensions. Within each
problem, we vary the minority-group size, yielding different
amplitudes $\psi := \sigma_{\max}(\nabla^2 L_0)$.
We then measure the entry times of gradient descent (GD) and sign
descent (SD), and fit their ratio as a function of $\psi$ and $d$
using log--log regression. This yields the following empirical
scaling relations for Superconductor (left) and YearPredictionMSD
(right):
\[
\timeentry(\mathrm{GD})/\timeentry(\mathrm{SD})
\propto d^{0.65}/\psi^{0.90},
\qquad
\timeentry(\mathrm{GD})/\timeentry(\mathrm{SD})
\propto d^{0.51}/\psi^{0.37}.
\]
These estimates show that the entry time of gradient descent effectively grows faster than that of Sign Descent in high-dimensional ($d>>1$) and highly imbalanced ($\psi<<1$) regimes. Detailed experimental protocols are provided in \Cref{sec-appendix-linear-regression}.\\
\textbf{Hidden subclass classification - Forest Covertype.} We predict forest cover type from $54$ cartographic variables ($\sim$$581,000$ samples). Following \cite{svmtorch}, we use balanced binary supervision (Lodgepole Pine vs.\ all other classes). In this formulation, the \textit{Aspen} class forms the hidden minority, comprising just $1.6\%$ of the dataset. Our results demonstrate that AdamW's early learning advantage over SGD is significantly magnified for this minority class ($63.93\%$ vs. $51.40\%$) compared to the majority class ($90.99\%$ vs $88.25\%$). This phenomenon also occurs below. \\
\textbf{Catch-up time with CIFAR-10.} To facilitate direct comparison with prior work, we replicate the experimental setup of \cite{bachoc2025majorityrulesminorityloses}, wherein the \textit{car} class is subsampled to form a minority group. 

\begin{table}[h]
\centering
\begin{tabular}{lccc}
\toprule
\textbf{Metrics} &  AdamW & SGD & Muon  \\
\midrule 
Average $L_0$ accuracy &  $91.82\%$ & $86.04\%$ & $96.48\%$  \\
Average $L_1$ accuracy &  $95.45\%$ & $92.92\%$ & $98.39\%$  \\
\bottomrule
\end{tabular}
\caption{Average accuracy over iterations $1-30$. This evaluation window was chosen because both $L_0$ and $L_1$ accuracies reach $99\%$ after iteration 40 in our simulations.}
\end{table}

\textbf{Language modelling - SmallGPT2 on WikiText-103.} Adopting the setup of \citet{kunstner2024heavytailed}, we train a SmallGPT-2 model from scratch with AdamW, SGD, and Muon, tracking both majority and minority (the $5\%$ rarest tokens) losses (\Cref{fig:smallgpt2}). 

While the majority loss steadily decreases across all optimizers, the minority loss trajectories diverge significantly. The AdamW and Muon minority losses initially spike before decreasing, whereas the SGD minority loss strictly increases throughout the experiment. This behavior corroborates our theoretical finding that standard gradient dynamics inherently delay optimization of the minority signal $L_0$. The setting of the experiments and detailed results are given in \Cref{tab-llm-hyperparameters}.

\begin{figure}[h]
    \centering
    \includegraphics[width=0.4\linewidth]{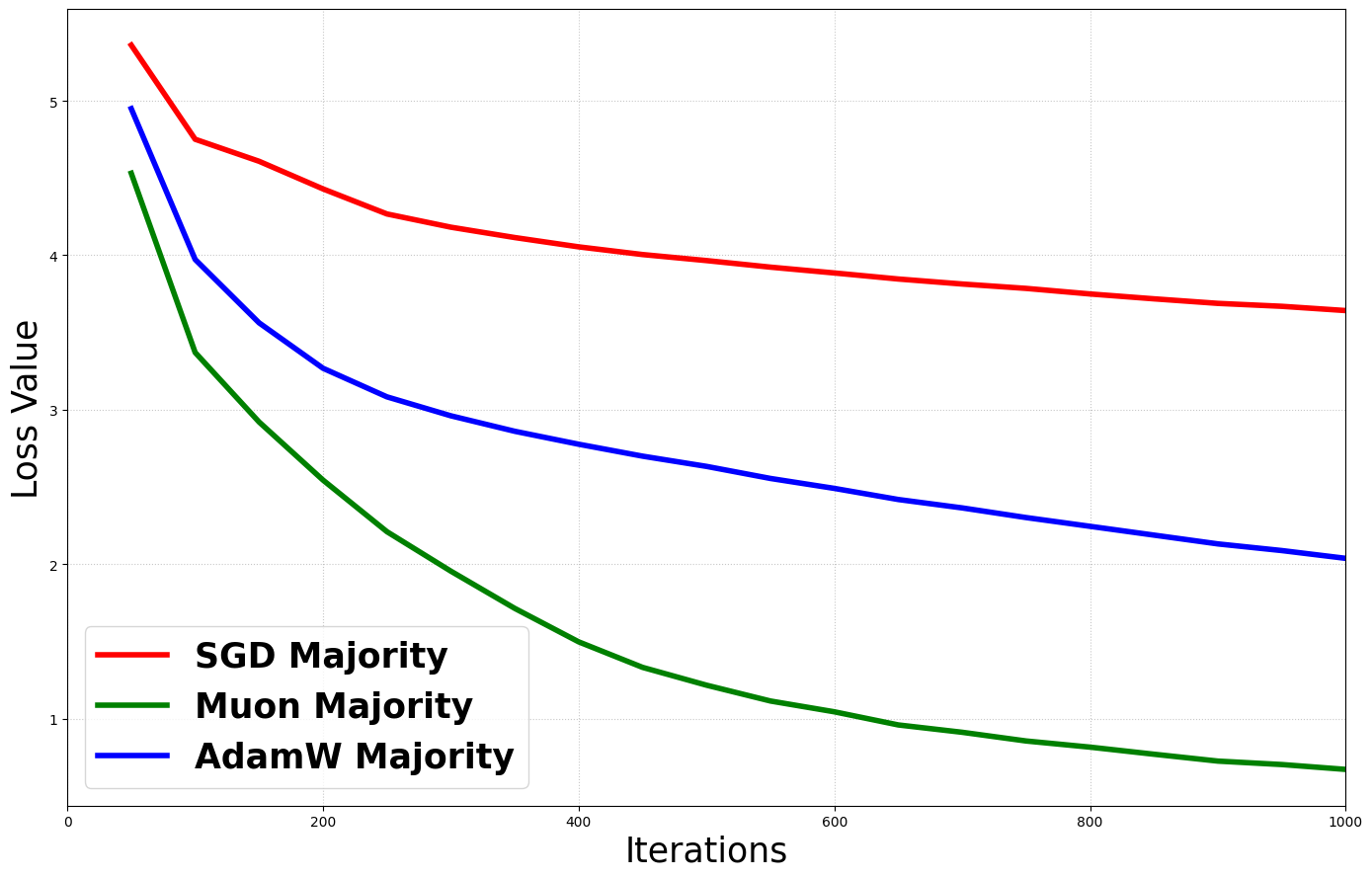}
    \includegraphics[width=0.4\linewidth]{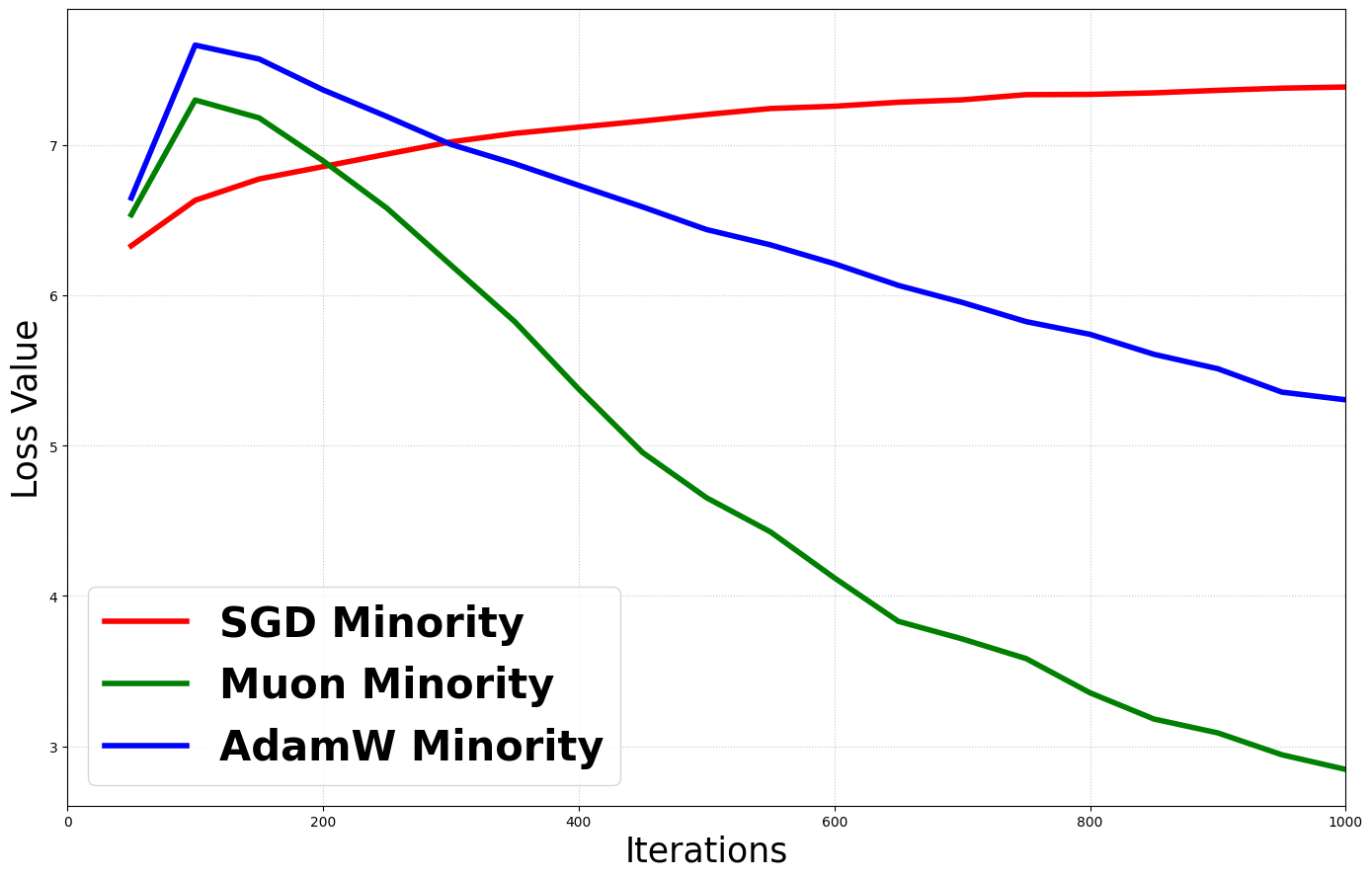}
    \caption{Evolution of training loss of SmallGPT2 from scratch on WikiText-103.}
    \label{fig:smallgpt2}
\end{figure}

\newpage
\subsection*{AI use statement}
In this work, we used generative AI tools to formulate mathematical claims, to assist in the writing of proofs, to provide critical ingredients for proving mathematical claims and to assist with translation. We have not used generative AI tools to help develop theoretical models or conceptual frameworks, propose or refine hypotheses, design or provide feedback on research  methodology or experiments, implement methods or interpret results. Generating synthetic data sets, cleaning and reformatting datasets to support qualitative and thematic data analysis is not applicable to this work.
Additionally, we used generative AI tools for transcribe recording of research material, summarizing or analyzing existing literature and editing the paper to improve readability. We have reviewed all AI-assisted work. LLM-generated code was verified and tested for correctness. We take responsibility for the final content of this work, including text, claims or artifacts produced with the aid of generative AI.

\subsection*{Funding}

This work was supported by the ANR project Regul IA and by the
Chairs TRIAL and UQPhysAI of the Toulouse ANITI AI Cluster.


\appendix

\section{Auxiliary results and comments}

\subsection{General comments}
We summarize here several key aspects of our approach, with the aim of delineating  what has been established and what remains to be investigated in a few points.

--- We introduce a majority--minority model under limited prior knowledge. Within this framework, it is difficult to guarantee that all minority components will eventually be learned, since this also depends on the initialization. Nevertheless, by detecting the onset of minority learning, we obtain an informative indicator of when minority components begin to emerge. Our key contributions are an equation for a minority awareness zone for most optimizers and estimates for entry times.

  This reconnects with the work of \cite{bachoc2025majorityrulesminorityloses} on the origin of bias of gradient methods and dominance of a majority learning region.

--- We work with idealized dynamics rather than genuine optimizers. The purpose is to retain the core vector fields governing the dynamics while stripping away secondary algorithmic effects.

--- We restrict attention to a class of functions arising from deep-learning losses, for which the relevant step sizes are small but vary over a moderate range, and for which the amplitude factor $\kappa$ introduced in \Cref{section:intro} remains modest. Thus, our analysis is intended to capture a regime that is simplified but still representative of the scales encountered in {\em deep learning} problems.

--- All else being equal, entry times are comparable across the dynamics we consider. They therefore provide an indication, for a given initial condition, of when majority dominance begins to cease.

--- Among the scaled optimizers considered, three clearly stand out: Adam, Muon and Newton's method. This suggests a promising direction, but the phenomenon remains to be understood more systematically.

Many aspects remain to be investigated. In particular:

--- extending the analysis from idealized dynamics to genuine optimizers;

--- developing more precise markers for minority detection,

--- understanding more systematically why certain scaled optimizers exhibit markedly different minority-learning behavior.

\subsection{Results on differential inclusions}

\begin{lemma}[Aubin--Cellina local existence theorem]
\label{Th:aubincell}
Let \(\Gamma:\mathbb{R}^d\rightrightarrows\mathbb{R}^d\) be a set-valued map
with nonempty compact convex values. Assume that \(\Gamma\) has a closed graph and is locally bounded. Then, for every \(\thetaInit\in\mathbb{R}^d\),
there exist \(T_{\max}>0\) and an absolutely continuous map
\(\theta:[0,T_{\max})\to\mathbb{R}^d\) such that
\[
\theta(0)=\thetaInit,
\qquad
\dot\theta(t)\in-\Gamma(\theta(t))
\quad\text{for a.e. }t\in[0,T_{\max}).
\]
This is a standard consequence of the existence theory for differential
inclusions; see, for instance, \cite[Chapter~2, Section~1]{aubin1984differential}.
\end{lemma}

\begin{corollary}[Global existence for norm-gradient systems]
\label{cor:global-existence-norm-gradient}
Under \Cref{ass-cost}, consider any of the following dynamical systems presented in \Cref{subsection:learning:with:descent}:
gradient descent, gradient clipping, sign descent, spectral descent and Newton descent (in this last case, assume in addition that \Eqref{ass-newton} holds). 
Then:

(i) any of this dynamical system has at least one solution, 

(ii) every maximal absolutely continuous solution of any of the above is global, that is,
\[
T_{\max}=+\infty.
\]
\end{corollary}

\begin{proof}
{\sc Existence}. The multifunction
\[
\ell\mapsto
\arg\max_{\|z\|\leq 1}\ell[z]
\]
has nonempty compact convex values,
because it is the set of maximizers of a linear form over the compact convex
unit ball of \(\|\cdot\|\). Since \(\theta\mapsto \mathrm{d}L(\theta)\) is
continuous, both
\[
\theta\mapsto
\ngradnorm{\|\cdot\|}{L(\theta)}
\qquad\text{and}\qquad
\theta\mapsto
\gradnorm{\|\cdot\|}{L(\theta)}
\]
are upper semicontinuous and locally bounded. Hence \Cref{Th:aubincell}
provides a local absolutely continuous solution for both gradient and normalized gradient methods. Note that for Newton's method we may  apply instead Arzela existence theorem (without uniqueness) and for gradient clipping it is more natural to apply Cauchy-Lipschitz.

{\sc Maximal solutions}. We argue for gradient dynamics in an arbitrary norm $\|\cdot\|$, the other methods follow by similar arguments.

\medskip

Let \(\theta:[0,T_{\max})\to\mathbb{R}^d\) be a maximal solution. For a.e. \(t\),
write
\[
\dot\theta(t)
\in -\gradnorm{\|\cdot\|}{L(\theta(t))},\] so
\begin{equation}
 \label{eq-descent-grad}
\frac{d}{dt} L(\theta(t))
= \mathrm{d}L(\theta(t))[\dot \theta(t)]
=
-\|\mathrm{d}L(\theta(t))\|_{\ast}^2.
\end{equation}
As a consequence, for all $t\in [0,T_\max)$ 
$$\int_0^t\|\mathrm{d}L(\theta(s))\|_{\ast}^2
ds
\leq L(\thetaInit)-L_\star.$$
Thus if $T_\max$ was finite we would have a constant $C>0$ such that
$$\int_0^{T_\max}\|\dot \theta(t)\|_{2}^2
dt
\leq C \int_0^{T_\max}\|\mathrm{d}L(\theta(t))\|_{\ast}^2dt \leq 
C \left( L(\thetaInit)-L_\star \right).$$
By Cauchy-Schwarz inequality, for $0 \le s < t \le T_\max$,
$$\|\theta(t)-\theta(s)\|_2=\int_s^{t}
\|
\dot \theta(\tau)
\|_2
d\tau\leq \sqrt{|t-s|}\sqrt{\int_0^{t} \|\dot \theta(\tau)\|^2_2d\tau}\leq \sqrt{C |t-s|}\sqrt{L(\thetaInit)-L_\star}.$$ 
Hence we see that $\theta$ has a limit at $t=T_\max$ which contradicts the maximality of $T_\max$. 

A similar reasoning provides that $T_\max=+\infty$ for Newton's method or gradient clipping. In the normalized case, the proof is even more direct because the velocity is globally bounded and the Cauchy property derived above boils down to $$\|\theta(t)-\theta(s)\|\leq |t-s|$$
for all $t,s$ in $[0,T_\max)$.

In all cases uniqueness is not required.
\end{proof}

\paragraph{Absence of finite time convergence.}
While normalized gradient descent can lead in a finite time towards a critical point, this is not possible for the normal gradient descent despite multivaluedness.

\begin{proposition}[Absence of finite time convergence for the normal gradient descent]
    Under \Cref{ass-cost}, let $\theta$ be an absolutely continuous solution to 
    $$\dot \theta(t)\in-\gradnorm{\|\cdot\|}{L(\theta(t))}$$
with $dL(\thetaInit)\neq 0$. Then for all time $t>0$, $dL(\theta(t))\neq 0$.
\end{proposition}
\begin{proof}
Suppose, toward a contradiction, that there exists $t_0>0$ such that
\[
    dL(\theta(t_0))=0, \mbox{i. e., }    \nabla L(\theta(t_0))=0 .
\]

By the definition of the norm-gradient and norm equivalence in finite
dimension, there exists a constant $C>0$ such that for a.e. $t$,
\[
    \|\dot\theta(t)\|_2
    \leq C \|\nabla L(\theta(t))\|_2 .
\]
Moreover, since $L$ is $C^2$ and $\theta([0,t_0])$ is compact, $\nabla L$
is Lipschitz on a neighbourhood of $\theta([0,t_0])$. Let $\Lambda>0$ be
a corresponding Lipschitz constant. 

For every $t\in[0,t_0]$, using $\nabla L(\theta(t_0))=0$, we get
\[
\|\nabla L(\theta(t))\|_2
    = \|\nabla L(\theta(t))-\nabla L(\theta(t_0))\|_2 
    \leq \Lambda \|\theta(t)-\theta(t_0)\|_2 
    \leq \Lambda \int_t^{t_0} \|\dot\theta(s)\|_2\,ds 
    \leq \Lambda C \int_t^{t_0}
        \|\nabla L(\theta(s))\|_2\,ds.\]
Define
\[
    g(t):=\|\nabla L(\theta(t))\|_2 .
\]
Then \(g\geq 0\) is continuous and satisfies
\[
    g(t)\leq \Lambda C \int_t^{t_0} g(s)\,ds
    \qquad\text{for all }t\in[0,t_0].
\]
Hence
\[
    g(t_0 - (t_0-t))\leq \Lambda C \int_0^{t_0-t} g(t_0-s)\,ds
    \qquad\text{for all }t\in[0,t_0].
\]
Thus, by Gronwall's lemma, \(g(t)=0\) for all
\(t\in[0,t_0]\). In particular, $\nabla L(\theta(0))=0,$ that is,
    $dL(\theta_{\rm init})=0$ which is impossible.
\end{proof}
\subsection{Bounds}
We recall the definition of the dual norm
\[ \|\ell\|_{\ast}
:=
\max_{\|z\|\leq1}\ell[z].\]
\begin{lemma}[Smoothness gap bounds for an arbitrary norm]
\label{lem-smooth-fct}
Let \(\|\cdot\|\) be a norm on \(\mathbb{R}^d\), and let
\(\|\cdot\|_{\ast}\) be the corresponding operator norm on linear forms. Consider
\(h\in \mathcal{C}^1(\mathbb{R}^d,\mathbb{R})\) which is \(L\)-smooth with
respect to \(\|\cdot\|\), i.e, with $L>0$,
\begin{equation} \label{eq:smoothness:inequality}
\forall x,y \in \mathbb{R}^d,\qquad
h(x)-h(y)-\mathrm{d}h(y)[x-y]\leq \frac{L}{2}\|x-y\|^{2}.
\end{equation}
Assume that \(h\) reaches its minimum at \(\theta_\star\). Then, for every
\(\theta\in\mathbb{R}^d\),
\[
\frac{1}{2L}\|\mathrm{d}h(\theta)\|_{\ast}^2
\leq
h(\theta)-h(\theta_\star)
\leq
\frac{L}{2}\|\theta-\theta_\star\|^{2}.
\]
\end{lemma}
\begin{proof}
For the upper bound, apply the smoothness inequality, \Eqref{eq:smoothness:inequality}, with \(x=\theta\) and
\(y=\theta_\star\). Since \(\theta_\star\) is a minimizer and \(h\in C^1\), one
has \(\mathrm{d}h(\theta_\star)=0\), so
\[
h(\theta)-h(\theta_\star)\leq \frac{L}{2}\|\theta-\theta_\star\|^2.
\]

For the lower bound, fix \(\theta\in\mathbb{R}^d\). By compactness of the unit
\(\|\cdot\|\)-ball and continuity of \(\mathrm{d}h(\theta)\), there exists
\(v_\theta\in\mathbb{R}^d\) such that
\[
\|v_\theta\|=1,
\qquad
\mathrm{d}h(\theta)[v_\theta]=\|\mathrm{d}h(\theta)\|_{\ast}.
\]
Apply the smoothness inequality with
\[
y=\theta,
\qquad
x=\theta-\frac{\|\mathrm{d}h(\theta)\|_{\ast}}{L}v_\theta.
\]
Then
\begin{align*}
h\!\left(\theta-\frac{\|\mathrm{d}h(\theta)\|_{\ast}}{L}v_\theta\right)
&\leq
h(\theta)
-\frac{\|\mathrm{d}h(\theta)\|_{\ast}}{L}\mathrm{d}h(\theta)[v_\theta]
+\frac{L}{2}\frac{\|\mathrm{d}h(\theta)\|_{\ast}^2}{L^2}\|v_\theta\|^2\\
&=
h(\theta)-\frac{1}{2L}\|\mathrm{d}h(\theta)\|_{\ast}^2.
\end{align*}
Since \(h(\theta_\star)\leq h(x)\) for every \(x\), this yields
\[
\frac{1}{2L}\|\mathrm{d}h(\theta)\|_{\ast}^2
\leq
h(\theta)-h(\theta_\star).
\]
\end{proof}

\subsection{Extra-material}

\begin{proposition}[Subdifferential of the dual  norm]\label{sousdiff}
Let \(E\) be a finite-dimensional real vector space, and let
\(\|\cdot\|\) be a norm on \(E\). Its dual norm is defined by
\[
\|V\|_*=\max_{\|Z\|\leq 1}\langle V,Z\rangle,
\]
where $\langle \cdot,\cdot \rangle$ is the duality bracket.\footnote{In our setting $E$ is identified with its dual with the Euclidean scalar product (Frobenius-like for matrices). So, after identification the bracket is actually the standard scalar product.}
Then,
$ \displaystyle
\partial \|V\|_*
=
\operatorname*{argmax}_{\|Z\|\leq 1}\langle V,Z\rangle
$, for every $V$ in \(E^*\).
Equivalently,
\[
Z\in\partial \|V\|_*
\quad\Longleftrightarrow\quad
\|Z\|\leq 1
\ \text{ and }\
\langle V,Z\rangle=\|V\|_*.
\]
\end{proposition}

\begin{proof}
Recall that the subdifferential of the convex function
\(V\mapsto \|V\|_*\) is defined by
\[
Z\in\partial \|V\|_*
\]
if and only if, for every \(Y\in E^*\),
\[
\|Y\|_*\geq \|V\|_*+\langle Y-V,Z\rangle.
\]
We prove both inclusions. First, assume
\[
Z\in \operatorname*{argmax}_{\|X\|\leq 1}\langle V,X\rangle,
\]
so that
\[
\|Z\|\leq 1
\mbox{ and }
\langle V,Z\rangle=\|V\|_*.
\]
For any \(Y\in E^*\), by definition of the dual norm,
\begin{equation}\label{eq:boundsub}
\langle Y,Z\rangle\leq \|Y\|_* \|Z\|\leq \|Y\|_*.
\end{equation}
Using \(\langle V,Z\rangle=\|V\|_*\), we get
\[
\langle Y,Z\rangle
=
\langle V,Z\rangle+\langle Y-V,Z\rangle
=
\|V\|_*+\langle Y-V,Z\rangle.
\]
Hence by \Eqref{eq:boundsub}
\[
\|Y\|_*
\geq
\|V\|_*+\langle Y-V,Z\rangle, \mbox{ i.e., }
Z\in\partial \|V\|_*.
\]

Conversely, assume
\[
Z\in\partial \|V\|_*.
\]
Then, for every \(Y\in E^*\),
\[
\|Y\|_*
\geq
\|V\|_*+\langle Y-V,Z\rangle.
\]
Taking \(Y=0\), we obtain
$
\langle V,Z\rangle\geq \|V\|_*.
$ On the other hand, the dual norm always satisfies
$
\langle V,Z\rangle\leq \|V\|_* \|Z\|.
$

To conclude we show that \(\|Z\|\leq 1\). Let \(Y=tH\), where \(t>0\) and
\(H\in E^*\). The subgradient inequality gives
\[
t\|H\|_*
\geq
\|V\|_*+\langle tH-V,Z\rangle.
\]
That is,
\[
t\|H\|_*
\geq
\|V\|_*-\langle V,Z\rangle+t\langle H,Z\rangle.
\]
Divide by \(t\):
\[
\|H\|_*
\geq
\frac{\|V\|_*-\langle V,Z\rangle}{t}+\langle H,Z\rangle.
\]
Letting \(t\to+\infty\), we get
\[
\langle H,Z\rangle\leq \|H\|_*
\qquad
\text{for every }H\in E^*.
\]
Taking the supremum over all \(H\) with \(\|H\|_*\leq 1\), we obtain
\[
\|Z\|\leq 1,
\]
because \(\|\cdot\|\) is the dual norm of \(\|\cdot\|_*\) in finite dimension.

Therefore
\[
\|Z\|\leq 1.
\]
By definition of \(\|V\|_*\), this implies
\[
\langle V,Z\rangle\leq \|V\|_*.
\]
Together with the earlier inequality
\[
\langle V,Z\rangle\geq \|V\|_*,
\]
we conclude that
\[
\langle V,Z\rangle=\|V\|_*.
\]
Hence
\[
Z\in \operatorname*{argmax}_{\|X\|\leq 1}\langle V,X\rangle.
\]
This proves the equality $ \displaystyle
\partial \|V\|_*
=
\operatorname*{argmax}_{\|Z\|\leq 1}\langle V,Z\rangle
$.
\end{proof}

\begin{remark}
\label{rem:strict:ineq}
In order to have the convenient variational representation in \Eqref{eq:calQ-multivalued} of $\Zmin$, 
our \Cref{def-Maj} differs slightly from
the corresponding definition in 
\cite{bachoc2025majorityrulesminorityloses} (where only univalued $V_L$ are considered), by
using a strict inequality. Observe that $\Zmaj(L_0)$ could be empty in some cases ---for instance when $L_0 = 0$ or $L_0 = - \epsilon L$ with $\epsilon>0$ small.
\end{remark}

\subsection{Spectral Descent}
\label{section:appendix-spectral-Descent}

\begin{definition} \label{def:compact:eigenvalue:decomposition}
Let $a,b \in \mathbb{N}$ and $M$ be an $a \times b$ non-zero matrix. A compact eigenvalue decomposition of $M$ is a decomposition $M = U \Sigma V^\top$ where $U$ is $a \times r$, with $r \in \mathbb{N}$, $r \le \min(a,b)$, with $U^\top U = I_r$, where $\Sigma$ is $r \times r$ and diagonal with non-zero diagonal values, and where $V$ is $b \times r$ with $V^\top V = I_r$.
\end{definition}

Assume the parameter space is of the form $\mathbb{R}^{a\times b}$. The differential is
represented by the matrix gradient through
\[
\mathrm{d}L(\Theta)[H]=\operatorname{tr}\bigl(\nabla L(\Theta)^\top H\bigr)
\qquad \forall \Theta \in\mathbb{R}^{a\times b}
\qquad \forall H\in\mathbb{R}^{a\times b}.
\] 
The  matrix Spectral norm (w.r.t. to Euclidean norm) is denoted by $\|\cdot\|_{\rm sp}$, its dual norm --the nuclear norm-- is denoted by 
$\|\cdot\|_{\mathrm{nuc}}$. Let us denote by $\Spec=\partial \|\cdot\|_{\mathrm{nuc}}$ the subdifferential mapping of the nuclear norm. From classical results, see \Cref{sousdiff}, 
one has, 
\begin{equation}
\Spec(M)= \argmax_{\|Z\|_{\rm sp}\leq 1}\operatorname{tr}(M^\top Z) \mbox{ for all $M$ in $\mathbb{R}^{a\times b}$,}
\end{equation}
and, using \cite[Example 2, Section 2]{watson}, if $M=U\Sigma V^\top$ is a compact singular value decomposition (see \Cref{def:compact:eigenvalue:decomposition}), then
\[
\Spec(M)
=
\left\{
UV^\top+W:
U^\top W=0,\ WV=0,\ \|W\|_{\mathrm{sp}} \leq 1
\right\}.
\]

Therefore, by definition of normalized gradients,
\begin{equation}
\label{eq:SM:tr}
\ngradnorm{\|\cdot\|_{\mathrm{sp}}}{ L(\Theta)}
=
\Spec (\nabla L(\Theta)) .
\end{equation}

Spectral descent is the corresponding normalized spectral-norm descent:
\begin{equation}
\label{eq:operator-norm-gradient}
\dot\Theta(t)\in
-\ngradnorm{\|\cdot\|_{\mathrm{sp}}}{ L(\Theta(t))}.
\end{equation}

Since \(\Theta\mapsto\Spec(\nabla L(\Theta))\) has nonempty compact convex and bounded values and
closed graph,
standard existence results for differential inclusions again apply to \Eqref{eq:normalized:spec:desc} (\Cref{Th:aubincell,cor:global-existence-norm-gradient}).

Note that the (unnormalized) spectral-norm descent would be
\[ \dot\Theta(t)\in
-\gradnorm{\|\cdot\|_{\mathrm{sp}}}{ L(\Theta(t))} \Leftrightarrow  \dot\Theta(t)\in
- \| \nabla L(\Theta(t))\|_{\rm nuc}\,
\ngradnorm{\|\cdot\|_{\rm sp}}{ L(\Theta(t))} \]
with 
\[ v \in -\gradnorm{\|\cdot\|_{\mathrm{sp}}}{ L(\Theta)} \implies \| v\|_{\rm sp} = \|\nabla L(\Theta)\|_{\mathrm{nuc}}. \]


\subsection{Majority adverse and Minority awareness zones}

Our definition of a majority adverse zone differs slightly from the one in \cite{bachoc2025majorityrulesminorityloses} because of the strict inequality (\Cref{rem:strict:ineq}). As a result, in our setting the minority aware zone $\Zmin$ typically has nonempty interior, see \Cref{cor-non-emptyness-robust-zones}. This is a benefit of its definition, compared to $\Zmaj(L_0)$ which can be empty for some choices of $L_0$ and non-empty for other choices.

\medskip

The next lemma shows that for an optimizer $V_L$ satisfying Equations \ref{eqdescent} and \ref{eqdescent1}, 
membership in the majority adverse zone forces a strict decrease of the
minority loss along the trajectory. 
A proof is given in \Cref{sec-appendix-proof-section-3}.

\begin{lemma}
\label{lemma-majority-increase}
Let $\theta$ be a solution to
\Eqref{eq-dynamical-system}. Assume that $V_L$ is a descent field, i.e. Equations \ref{eqdescent} and \ref{eqdescent1}  hold and $\Zmaj(L_0) \neq \emptyset$ .
Then, for every $L_0\in\Kmin$, for a.e. $t\geq 0$,
$\displaystyle
\theta(t)\in\Zmaj(L_0)
\quad\Longrightarrow\quad
\frac{d}{dt}L_0(\theta(t))<0.$
\end{lemma}

The next theorem shows uniform majority dominance out of $\Zmin$.
\begin{theorem}[Uniform majority dominance:
the multivalued case]
\label{thm:uniform-majority-dominance-general}
Under Assumptions~2.1 and~3.1, for every
$\theta\in\mathbb{R}^d$,
\[
\theta\notin\Zmin
\quad\Longleftrightarrow\quad
\exists v_\theta\in V_L(\theta),\
\forall L_1\in\Kmax(L),\
dL_1(\theta)[-v_\theta]\leq 0.
\]
Thus, outside $\Zmin$, there exists a common admissible
update that is majority-favouring to first order for every
admissible decomposition.

If $V_L$ is single-valued, this is the actual update.
In particular, let $\theta:I\to\mathbb{R}^d$ be an
absolutely continuous solution of
$\dot\theta(t)=-V_L(\theta(t))$ on an interval $I$,
with $\theta(t)\notin\Zmin$ for every $t\in I$.
Then, for every $L_1\in\Kmax(L)$,
\[
\frac{d}{dt}L_1(\theta(t))\leq 0
\qquad\text{for almost every }t\in I,
\]
and $L_1\circ\theta$ is nonincreasing on $I$.
\end{theorem}

\begin{proof}
Fix $\theta$ and write $V=V_L(\theta)$. Define
\[
C_\theta
:=
\overline{
\{\nabla L_1(\theta):L_1\in\Kmax(L)\}
}.
\]
The set $\Kmin$ is nonempty and convex, and so is
$\Kmax(L)$. Moreover, for every $L_1\in\Kmax(L)$,
\[
\|\nabla L_1(\theta)-\nabla L(\theta)\|_2
\leq \psi(\|\theta\|_2+D).
\]
Hence $C_\theta$ is nonempty, convex, and compact.
The set $V$ has the same properties by Assumption~3.1.

Suppose first that $\theta\notin\Zmin$. By the definitions
of $\Zmin$ and $\Zmaj(L_0)$,
\[
\max_{v\in V}\langle v,\nabla L_1(\theta)\rangle
\geq 0
\qquad\text{for every }L_1\in\Kmax(L).
\]
Since $g\mapsto\max_{v\in V}\langle v,g\rangle$ is
continuous, this inequality extends to every $g\in C_\theta$.
The bilinear minimax theorem therefore gives
\[
0
\leq
\min_{g\in C_\theta}\max_{v\in V}\langle v,g\rangle
=
\max_{v\in V}\min_{g\in C_\theta}\langle v,g\rangle.
\]
The last maximum is attained. Any maximizer $v_\theta$
satisfies
\[
\langle v_\theta,g\rangle\geq 0
\qquad\text{for every }g\in C_\theta.
\]
In particular,
$dL_1(\theta)[-v_\theta]\leq 0$
for every $L_1\in\Kmax(L)$.

Conversely, suppose such a common direction $v_\theta$
exists. For every $L_0\in\Kmin$, with $L_1=L-L_0$,
\[
\sup_{v\in V_L(\theta)}
\langle v,\nabla L_1(\theta)\rangle
\geq
\langle v_\theta,\nabla L_1(\theta)\rangle
\geq 0.
\]
Thus $\theta\notin\Zmaj(L_0)$ for every $L_0\in\Kmin$,
and therefore $\theta\notin\Zmin$.

Finally, when $V_L$ is single-valued, the common direction
is necessarily $v_\theta=V_L(\theta)$. Along the trajectory
in the statement, the chain rule yields
\[
\frac{d}{dt}L_1(\theta(t))
=
dL_1(\theta(t))[-V_L(\theta(t))]
\leq 0
\qquad\text{for almost every }t\in I.
\]
Integrating proves the claimed monotonicity and, in
particular, \Cref{thm:uniform-majority-dominance}.
\end{proof}

\begin{lemma}[Minority activation]
\label{lem:minority-activation}
Under Assumptions~2.1 and~3.1, suppose that $V_L$ satisfies
the descent conditions~(5)--(6).
For every $\bar\theta\in\Zmin$, there exist an admissible
minority loss $L_0\in\Kmin$ and a neighbourhood $U$ of
$\bar\theta$ such that, writing $L_1=L-L_0$,
\[
dL_0(\theta)[-v]<0<dL_1(\theta)[-v],
\qquad
\forall\theta\in U,\quad \forall v\in V_L(\theta).
\]
Thus, every admissible update in $U$ strictly favors
the same minority loss while increasing its corresponding
majority loss, to first order.\footnote{In particular, if an absolutely continuous solution of
$\dot\theta(t)\in-V_L(\theta(t))$ satisfies
$\theta(t_0)\in\Zmin$, then there is a time neighborhood
of $t_0$, relative to the solution's domain, on which
one fixed admissible minority loss strictly decreases
and its corresponding majority loss strictly increases.}

\end{lemma}

\begin{proof}
Since $\bar\theta\in\Zmin$, the union defining $\Zmin$
provides $L_0\in\Kmin$ such that
$\bar\theta\in\Zmaj(L_0)$. With $L_1=L-L_0$, this means
\[
\max_{v\in V_L(\bar\theta)}
\langle v,\nabla L_1(\bar\theta)\rangle<0.
\]
This strict inequality persists in a neighbourhood $U$
of $\bar\theta$. Indeed, otherwise there would exist
$\theta_n\to\bar\theta$ and $v_n\in V_L(\theta_n)$ with
$\langle v_n,\nabla L_1(\theta_n)\rangle\geq0$.
Local boundedness gives a convergent subsequence of
$(v_n)$, and the closed-graph property places its limit
in $V_L(\bar\theta)$, contradicting the displayed
inequality by continuity of $\nabla L_1$.

Consequently, for every $\theta\in U$ and
$v\in V_L(\theta)$,
\[
dL_1(\theta)[-v]
=-\langle v,\nabla L_1(\theta)\rangle>0.
\]
The descent conditions give $dL(\theta)[-v]\leq0$.
Since $L_0=L-L_1$, it follows that
\[
dL_0(\theta)[-v]
=dL(\theta)[-v]-dL_1(\theta)[-v]<0.
\]

Finally, continuity of the trajectory gives a time
interval $I$ around $t_0$, within its domain, such that
$\theta(t)\in U$ for all $t\in I$.
The chain rule yields
\[
\frac{d}{dt}L_0(\theta(t))<0
\quad\text{and}\quad
\frac{d}{dt}L_1(\theta(t))>0
\qquad\text{for almost every }t\in I.
\]
Absolute continuity and integration establish the
claimed strict monotonicity.
\end{proof}

\section{Proofs of \Cref{section:learning:process:majority,section:training:times} and lower bounds}


\subsection{Proofs of \Cref{section:learning:process:majority}}
\label{sec-appendix-proof-section-3}

\begin{proof}[Proof of \Cref{lemma-majority-increase}]
Let $L_1=L-L_0$. Since $\theta$ is absolutely continuous and $L,L_0,L_1$ are
$C^1$, the chain rule holds for a.e. $t\geq 0$. Fix such a time and write
\[
\dot\theta(t)=-v(t),
\qquad
v(t)\in V_L(\theta(t)).
\]
Assume moreover that \(\theta(t)\in \Zmaj(L_0)\). By definition of
\(\Zmaj(L_0)\),
\[
\langle v(t),\nabla L_1(\theta(t))\rangle<0.
\]
Using \(L_0=L-L_1\), \(\dot\theta(t)=-v(t)\), and the descent property, Equations
\ref{eqdescent} and \ref{eqdescent1}, we obtain
\begin{align*}
\frac{d}{dt}L_0(\theta(t))
&=
\left\langle\nabla L_0(\theta(t)),\dot\theta(t)\right\rangle\\
&=
-\left\langle v(t),\nabla L(\theta(t))\right\rangle
+\left\langle v(t),\nabla L_1(\theta(t))\right\rangle.
\end{align*}
The first term is nonpositive by Equations \ref{eqdescent} and \ref{eqdescent1}, and the second term is
strictly negative. Hence \(\frac{d}{dt}L_0(\theta(t))<0\).
\end{proof}

\begin{proof}[Proof of \Cref{lemma:pointwise-envelope-K}]
Fix $\theta\in\mathbb{R}^d$ and set $r:=\|\theta\|_2$. The inclusion
``$\subseteq$'' is exactly the gradient bound in \Eqref{def-K}. Conversely, fix
$\xi\in\mathbb{R}^d$ with $\|\xi\|_2\leq \psi(r+D)$. If $r+D=0$, then
$\xi=0$ and the zero function realizes it. Assume now that $r+D>0$ and split
\[
\xi_r:=\frac{r}{r+D}\xi,
\qquad
\xi_D:=\frac{D}{r+D}\xi .
\]
Then $\|\xi_r\|_2\leq \psi r$ and $\|\xi_D\|_2\leq \psi D$. Let
$A:\mathbb{R}^d\to\mathbb{R}^d$ be linear. If $r=0$ or $\xi_r=0$, take $A=0$.
Assume now that $r>0$ and $\xi_r\neq0$, and set
\[
u:=\frac{\theta}{r},
\qquad
s:=\frac{\xi_r}{\|\xi_r\|_2}.
\]
Then \(\|u\|_2=\|s\|_2=1\). If \(s=u\), set \(B:=I\). Otherwise define
\[
B
:=
I-2\frac{(u-s)(u-s)^\top}{\|u-s\|_2^2}.
\]
The matrix
\[
P:=\frac{(u-s)(u-s)^\top}{\|u-s\|_2^2}
\]
is the orthogonal projector onto \(\operatorname{span}\{u-s\}\), so \(B=I-2P\)
is symmetric and satisfies \(B^2=I\). Hence \(\|B\|_{\rm sp}=1\). Moreover,
using \(\|u\|_2=\|s\|_2=1\),
\[
\|u-s\|_2^2=2(1-\langle u,s\rangle),
\]
therefore
\[
Bu
=
u-2(u-s)\frac{\langle u-s,u\rangle}{\|u-s\|_2^2}
=
u-(u-s)
=
s.
\]
Set
\[
A:=\frac{\|\xi_r\|_2}{r}\,B.
\]
Then \(A\) is symmetric, \(\|A\|_{\rm sp}=\|\xi_r\|_2/r\leq\psi\), and
\[
A\theta
=
\frac{\|\xi_r\|_2}{r}Br u
=
\|\xi_r\|_2 s
=
\xi_r.
\]
Now define
\[
L_{0,\xi}(z):=\frac12\langle Az,z\rangle+\langle \xi_D,z\rangle .
\]
Then $L_{0,\xi}$ is $C^2$, \(\nabla L_{0,\xi}(\theta)=\xi_r+\xi_D=\xi\), and,
for every $z\in\mathbb{R}^d$,
\[
\|\nabla L_{0,\xi}(z)\|_2
\leq
\|A\|_{\rm sp}\|z\|_2+\|\xi_D\|_2
\leq
\psi(\|z\|_2+D).
\]
Thus $L_{0,\xi}\in\Kmin$, proving the reverse inclusion in
\Eqref{eq:pointwise-envelope-K}.
\end{proof}

\begin{proof}[Proof of \Cref{thm-calQ}]

Fix $\theta\in\mathbb{R}^d$ and set $\rho(\theta):=\psi(\|\theta\|_2+D)$. 
By the definitions of $\Qrob$ and $Q_{L_0}$,
\[
\Qrob(\theta)
=
\inf_{L_0\in\Kmin}
\sup_{v\in V_L(\theta)}
\left\langle v,\nabla(L-L_0)(\theta)\right\rangle .
\]
At this fixed $\theta$, the expression depends on $L_0$ only through the vector
$\xi=\nabla L_0(\theta)$. The pointwise envelope identity
\Eqref{eq:pointwise-envelope-K}, given by \Cref{lemma:pointwise-envelope-K},
therefore gives
$\displaystyle 
\Qrob(\theta)
=
\inf_{\|\xi\|_2\leq \rho(\theta)}
\sup_{v\in V_L(\theta)}
\left\langle v,\nabla L(\theta)-\xi\right\rangle,$
which proves the following equation
\begin{align*}
\Qrob(\theta) =
\inf_{\|\xi\|_2\leq \psi(\|\theta\|_2+D)}
\sup_{v\in V_L(\theta)}
\left\langle v,\nabla L(\theta)-\xi\right\rangle .
\end{align*}
Let us observe that $\displaystyle C(\theta) := \{ \xi \in \mathbb{R}^d, \quad \| \xi\|_2 \leq \psi(\| \theta\|_2 + D) \}$ is nonempty convex and compact. 
Consider the bilinear form
\begin{align*}
\phi_{\theta} :
\begin{cases}
 C(\theta) \times V_L(\theta)  &  \xrightarrow{} \mathbb{R} \\
 (\xi, v ) & \xrightarrow{} \langle v, \nabla L(\theta)- \xi \rangle. 
\end{cases}
\end{align*}
Using von Neumann's minimax Theorem (see also \cite{sion1958general}), we deduce 
\begin{align*}
\inf_{\xi \in C(\theta)} \sup_{v \in V_L(\theta)} \phi_{\theta}(\xi, v ) =  \sup_{v \in V_L(\theta)} \inf_{\xi \in C(\theta)} \phi_{\theta}(\xi, v ).
\end{align*}
Therefore, for $\theta\in\mathbb{R}^d$,
\begin{align*}
\Qrob(\theta) & =
\sup_{v\in V_L(\theta)} \inf_{\|\xi\|_2\leq \psi(\|\theta\|_2+D)}
\left\langle v,\nabla L(\theta)-\xi\right\rangle \\
& = \sup_{v\in V_L(\theta)} \Big( \langle v, \nabla L(\theta) \rangle + \inf_{\xi \in C(\theta)} \left( - \langle v, \xi \rangle \right) \Big) \\ 
& = \sup_{v\in V_L(\theta)} \Big(  \langle v, \nabla L(\theta) \rangle- \psi(\|\theta\|_2+D) \| v\|_2 \Big).
\end{align*}
Hence \Eqref{eq:calQ-gradient-multivalued} holds.
\Cref{lemma:robust-zone-calQ-main} yields the final equivalence in \Cref{thm-calQ}.
\end{proof}

\begin{corollary}[Inner minority aware Zones]
\label{cor-non-emptyness-robust-zones}
Under Assumptions \ref{ass-cost} and \ref{ass-VL} and \Eqref{eqdescent}, we know 
\[  \{ \theta \in \mathbb{R}^d, 0 < \| \nabla L(\theta)\|_2 < \psi(\|\theta \|_2 +  D)\} \subseteq  \Zmin.  \]
\end{corollary}
\begin{proof}
Let $\theta \in \mathbb{R}^d$ such that 
\[ 0 < \| \nabla L(\theta)\|_2 < \psi(\|\theta \|_2 +  D). \]
Since $\nabla L(\theta) \neq 0$, we know $0 \notin V_L(\theta)$. Therefore, since $V_L(\theta)$ is compact, we can define
\[ m(\theta) := \min_{v \in V_L(\theta)} \| v\|_2 > 0. \]
We know for all $v \in V_L(\theta)$,
\begin{align*}
\langle v, \nabla L(\theta)\rangle - \psi(\|\theta \|_2 +  D) \| v\|_2 & \leq \| v\|_2 \| \nabla L(\theta) \|_2 - \psi(\|\theta \|_2 +  D)\| v\|_2 \\
& = \| v\|_2 \Big( \| \nabla L(\theta) \|_2 - \psi(\|\theta \|_2 +  D) \Big) \\
& \le m(\theta) \Big( \| \nabla L(\theta) \|_2 - \psi(\|\theta \|_2 +  D) \Big).
\end{align*}
Therefore, from \Cref{thm-calQ}, we know :
\begin{align*}
\Qrob(\theta) =\sup_{v \in V_L(\theta)} \langle v, \nabla L(\theta)\rangle - \psi(\|\theta \|_2 +  D)  \| v\|_2 \leq m(\theta) \Big( \| \nabla L(\theta) \|_2 - \psi(\|\theta \|_2 +  D) \Big) < 0,
\end{align*}
which implies $\theta \in \Zmin$.
\end{proof}

\begin{corollary}[Robust zone for norm-gradient fields]
\label{cor-robust-zone-norm-gradient-fields}
Under the assumptions of \Cref{thm-calQ}, one has:
\begin{itemize}
\item If \(V_L(\theta)=\ngradnorm{\|\cdot\|}{L(\theta)}\), then
\[
\Qrob(\theta)
=
\|\mathrm{d}L(\theta)\|_{\ast}
-
\psi(\|\theta\|_2+D) m_{\|\cdot\|}(\theta),
\]
and therefore
\begin{equation}\label{eq:zmaj-grad:norm}
\theta\in\Zmin
\quad\Longleftrightarrow\quad
\|\mathrm{d}L(\theta)\|_{\ast}
<
\psi m_{\|\cdot\|}(\theta)(\|\theta\|_2+D).
\end{equation}
\item If \(V_L(\theta)=\gradnorm{\|\cdot\|}{L(\theta)}\), then
\[
\Qrob(\theta)
=
\|\mathrm{d}L(\theta)\|_{\ast}^2
-
\psi(\|\theta\|_2+D) m_{\|\cdot\|}(\theta)
\|\mathrm{d}L(\theta)\|_{\ast},
\]
and therefore
\begin{equation}\label{eq:zmaj-grad}
\theta\in\Zmin
\quad\Longleftrightarrow\quad
\|\mathrm{d}L(\theta)\|_{\ast}^2
<
\psi m_{\|\cdot\|}(\theta)
\|\mathrm{d}L(\theta)\|_{\ast}(\|\theta\|_2+D).
\end{equation}
\end{itemize}
\end{corollary}

Note that in Equations \ref{eq:zmaj-grad:norm} and \ref{eq:zmaj-grad},
the zones do not change. Indeed the strict negativity in \Eqref{eq:Zmaj-Q} implies that the zones do not contain critical points, so that we may divide \Eqref{eq:zmaj-grad} by $\|dL(\theta)\|_{\ast}$.

\begin{proof} Fix \(\theta \in \mathbb{R}^d\) and set
\(\rho:=\psi(\|\theta\|_2+D)\) and \(N:=\ngradnorm{\|\cdot\|}{L(\theta)}\).
For the normalized field in \Eqref{eq:zmaj-grad:norm}, 
by \Cref{thm-calQ}, and from the fact that $\mathrm{d}L(\theta)[N]=\{\|\mathrm{d}L(\theta)\|_{\ast}\}$ (see \Eqref{eq:l_op}), one has 
\begin{align*}
\Qrob(\theta)
& = \sup_{v\in N}  \Big( \langle v, \nabla L(\theta) \rangle- \psi(\|\theta\|_2+D) \| v\|_2 \Big) \\
& = \|\mathrm{d}L(\theta)\|_{\ast} -
\psi(\|\theta\|_2+D) \inf_{v\in N} \| v\|_2 \\
& = \|\mathrm{d}L(\theta)\|_{\ast} -
\psi(\|\theta\|_2+D) m_{\|\cdot\|}(\theta).
\end{align*}
For the norm-gradient field in \Eqref{eq:zmaj-grad}, using \Eqref{eq:l_op} and a change of variable argument,
\begin{align*}
\Qrob(\theta)
& = \sup_{v\in \|\mathrm{d}L(\theta)\|_{\ast} N}  \Big( \langle v, \nabla L(\theta) \rangle- \psi(\|\theta\|_2+D) \| v\|_2 \Big) \\
& = 
\|\mathrm{d}L(\theta)\|^2_{\ast} -
\psi(\|\theta\|_2+D) \inf_{v\in \|\mathrm{d}L(\theta)\|_{\ast} N} \| v\|_2 \\
& = 
\|\mathrm{d}L(\theta)\|^2_{\ast} -
\psi(\|\theta\|_2+D) m_{\|\cdot\|}(\theta) \|\mathrm{d}L(\theta)\|_{\ast}.
\end{align*}
The two characterizations of \(\Zmin\), Equations \ref{eq:zmaj-grad:norm} and \ref{eq:zmaj-grad}, follow from
\Cref{lemma:robust-zone-calQ-main}.
\end{proof}

Now we present different minority aware zones.

\paragraph{Gradient descent}
For gradient descent $V_L=\nabla L$,
\Cref{thm-calQ} gives
\begin{equation}
\label{eq:calQ:gd}
\Qrob_{\rm GD}(\theta)
=
\|\nabla L(\theta)\|_2^2-\psi(\|\theta\|_2+D)\|\nabla L(\theta)\|_2 .
\end{equation}
Therefore
\begin{equation}
\label{eq:robust-zone-gd}
\Zmin^{\rm GD}
=
\left\{\theta\in\mathbb{R}^d:
\|\nabla L(\theta)\|_2^2
<
\psi\|\nabla L(\theta)\|_2(\|\theta\|_2+D)
\right\}.
\end{equation}

\paragraph{Gradient clipping}
Using the notation of \Cref{subsection:learning:with:descent}, \Cref{thm-calQ} gives
\begin{equation}
\label{eq:calQ:clip}
\Qrob_{\rm clip}(\theta)
=
\| \Clip_{\tau_{\rm clip}}(\nabla L(\theta))\|_2 \left( \|\nabla L(\theta)\|_2 -
\psi(\|\theta\|_2+D) \right).
\end{equation}
Consequently,
\begin{equation}
\label{eq:robust-zone-clip}
\Zmin^{\rm clip}
=
\Zmin^{\rm GD}.
\end{equation}
Thus clipping merely changes velocities but not the minority aware zone.

\paragraph{Sign descent}
For the supremum norm (\(\|\cdot\|=\|\cdot\|_\infty\)), 
letting $|\supp(v)|$ be the number of non-zero components of a vector $v$,
we get
\begin{equation}
\label{eq:calQ:q-sd}
\Qrob_{\rm SD}(\theta)
=
\|\nabla L(\theta)\|_1
-
\psi\sqrt{|\supp(\nabla L(\theta))|}(\|\theta\|_2+D).
\end{equation}
Thus 
\begin{equation}
\label{eq:robust-zone-sd}
\Zmin^{\rm SD}
=
\left\{\theta\in\mathbb{R}^d:
\|\nabla L(\theta)\|_1
<
\psi\sqrt{|\supp(\nabla L(\theta))|}(\|\theta\|_2+D)
\right\}.
\end{equation}

\paragraph{Spectral descent}
Let $\nabla L(\Theta)=U\Sigma V^\top$ be a compact singular value decomposition, as in \Cref{subsection:learning:with:descent}. Combining
\Eqref{eq:operator-norm-gradient} with \Eqref{eq:calQ-gradient-multivalued}
gives
\begin{eqnarray}
\Qrob_{\rm Spec}(\Theta)
 & =& 
\|\nabla L(\Theta)\|_{\rm nuc}- \psi(\|\Theta\|_F+D) \inf_{Z\in \ngradnorm{\|\cdot\|_{\rm sp}}{ L(\Theta)}}
\|  Z \|_F \\
& = & 
\|\nabla L(\Theta)\|_{\rm nuc}-\psi(\|\Theta\|_F+D)\sqrt{\operatorname{rank}\nabla L(\Theta)}.\label{eq:calQ:spectral}
\end{eqnarray}
The full proof of \Eqref{eq:calQ:spectral} is below.
Thus
\begin{equation}
\label{eq:robust-zone-spectral}
\Zmin^{\rm Spec}
=
\left\{\Theta\in\mathbb{R}^{a \times b}:
\|\nabla L(\Theta)\|_{\mathrm{nuc}}
<
\psi\sqrt{\operatorname{rank}(\nabla L(\Theta))}(\|\Theta\|_F+D) \right\}.
\end{equation}

\paragraph{Newton descent}
For Newton descent, assuming in addition \Cref{ass-newton},
we have, by \Cref{thm-calQ}, 
\begin{equation}
\label{eq:calQ:newton}
\Qrob_{\rm N}(\theta)
=
\left\langle v_N(\theta),\nabla L(\theta)\right\rangle
-
\psi\|v_N(\theta)\|_2(\|\theta\|_2+D),
\end{equation}
with $v_N(\theta) := \nabla^2L(\theta)^{-1}\nabla L(\theta)$ for all $\theta$.

\begin{proof}[Proof of the expression of $\Qrob$ for Spectral Descent (\Eqref{eq:calQ:spectral})]
To establish the formula it suffices to apply \Cref{cor-robust-zone-norm-gradient-fields}, i.e. it suffices to compute  
\[
m_{\|\cdot\|}(\Theta):=
\operatorname{dist}\bigl(0,\ngradnorm{\|\cdot\|}{L(\Theta)} \bigr)=\operatorname{dist}\bigl(0,\partial \|\nabla L(\Theta)\|_{\rm nuc} \bigr).
\] Set
\(G=\nabla L(\Theta)\) and \(r=\mbox{rank } G\), then in view of the previous developments
\[
\partial \|G\|_{\rm nuc}
=
\{UV^\top+W:\ U^\top W=0,\ WV=0,\ \|W\|_{\rm sp}\leq 1\}.
\]
Hence
\[
\|UV^\top+W\|_F^2
=
\|UV^\top\|_F^2+\|W\|_F^2
=
r+\|W\|_F^2,
\]
where the cross term vanishes by the orthogonality constraints. Therefore
\[
\operatorname{dist}\bigl(0,\partial\|G\|_{\rm nuc}\bigr)
=
\min_{Z\in\partial\|G\|_{\rm nuc}}\|Z\|_F
=
\|UV^\top\|_F
=
\sqrt r.
\]

Consequently,
\begin{equation}
\Theta\in\Zmin^{\rm Spec}
\quad\Longleftrightarrow\quad
\|\nabla L(\Theta)\|_{\rm nuc}
<
\psi\sqrt{\operatorname{rank}(\nabla L(\Theta))}(\|\Theta\|_F+D).
\end{equation}
\end{proof}

\subsection{Proofs of \Cref{section:training:times}}
\label{sec-appendix-c-3}

\Cref{table-reference-entry-time} indicates the references for the upper bounds and lower bounds for each mechanism studied in \Cref{section:training:times}.
\begin{table}[H]
    \centering
    \begin{tabular}{lcc}
        \toprule
         \textbf{Dynamics} &  \textbf{Upper bounds} &  \textbf{Lower bounds}  \\
        \midrule
         Gradient descent & \Cref{sec:appendix-entry-time-applications} & \Cref{th:LB:min:awa}\\
        Gradient clipping & \Cref{thm-clip-entry-time-upper} & $\times $\\ 
        Sign descent & \Cref{sec:appendix-entry-time-applications} & \Cref{th:LB:min:awa:sign}\\
        Spectral descent & \Cref{sec:appendix-entry-time-applications} & \Cref{th:LB:min:awa:spec}  \\
        Newton descent & \Cref{thm-newton-entry-time-upper} & $\times $\\
        \bottomrule
    \end{tabular}
    \caption{References for the entry times in \Cref{section:training:times}.}
    \label{table-reference-entry-time}
\end{table}

\subsubsection{Preliminaries}

\begin{lemma}[Entry-time  for multivalued dynamics]
\label{lemma:before:tentry}
Under \Cref{ass-cost,ass-VL}, let $\theta$ be an absolutely continuous solution of
\Eqref{eq-dynamical-system} with $V_L$ satisfying the descent conditions, Equations \ref{eqdescent} and \ref{eqdescent1}. 
For all $t$ such that $\theta(t)$ is not in the minority aware zone, we have the following property, for all $L_0\in \Kmin$,
\[ \sup_{v \in V_L(\theta(t))} \langle v, \nabla L_1(\theta(t))\rangle  \geq 0. \] 
For almost all $t$ such that $\theta(t)$ is in the minority aware zone, there exists a decomposition $L=L_1+L_0$, with $L_0\in \Kmin$, $L_1\in\Kmax$, such that 
\begin{align*}
 \frac{d}{dt}L_1(\theta(t)) \geq - \sup_{v \in V_L(\theta(t))} \langle v, \nabla L_1(\theta(t))\rangle  >0, \quad \frac{d}{dt}L_0(\theta(t))<0 .
\end{align*}
\end{lemma}
\begin{proof}
For the first inequality, by definition, 
\begin{align*}
\theta(t) \not\in \Zmin \Leftrightarrow \inf_{L_0 \in \Kmin} Q_{L_0}(\theta(t)) = \inf_{L_0 \in \Kmin}  \sup_{v \in V_L(\theta(t))} \langle v, \nabla L_1(\theta(t)) \rangle \geq 0.
\end{align*}
For the second inequality,  denote by $v(t)\in V_L(\theta(t))$ the selected velocity, i.e., $\dot \theta(t)=-v(t)$. 
Fix a time \(t\) like in the assumptions at which $\theta$ is differentiable with derivative coinciding with $-v(t)$. Since \(\theta(t)\in\Zmaj(L_0)\) for some $L_0$, then 
$\langle v(t),\nabla L_1(\theta(t))\rangle<0.$
Thus
\[
\frac{d}{dt}L_1(\theta(t))
=
\left\langle\nabla L_1(\theta(t)),\dot\theta(t)\right\rangle
=
-\left\langle v(t),\nabla L_1(\theta(t))\right\rangle
\geq - \sup_{v \in V_L(\theta(t))} \langle v, \nabla L_1(\theta(t))\rangle 
>0.
\]
Since \(V_L\) satisfies \Eqref{eqdescent1},
one has 
$
\frac{d}{dt}L_0(\theta(t))
=
-\left\langle v(t),\nabla L(\theta(t))\right\rangle
+\left\langle v(t),\nabla L_1(\theta(t))\right\rangle
<0.
$
\end{proof}

\begin{lemma}[Lower bound of $m_{\|\cdot\|}$]
\label{lem-lower-bound-m}
Consider an arbitrary norm $\| \cdot\|$ on $\mathbb{R}^d$. We recall from \Cref{mtheta-cnorm} that
\[ c_{\| \cdot \|} := \inf_{v \in \mathbb{R}^d /\{ 0\}} \frac{\|v\|_2}{\| v\|}. \]
Then we know that for all $\theta$ such that $\nabla L(\theta) \neq 0$,
\[ m_{\|\cdot\|}(\theta) \geq  c_{\| \cdot \|}  > 0. \] 
\end{lemma}
\begin{proof}
Let $\theta \in \mathbb{R}^d$ such that $\nabla L(\theta) \neq 0$. Recall that
\begin{align*}
m_{\|\cdot\|} (\theta) := \inf_{v\in \ngradnorm{\|\cdot\|}{L(\theta)}}\|v\|_2.
\end{align*}
Since $\nabla L(\theta)\neq 0$, the linear function
\begin{align*}
\phi : \begin{cases}
\{ v, \| v\|\leq 1 \} &\xrightarrow{} \mathbb{R}  \\
z & \mapsto \langle \nabla L(\theta),z \rangle 
\end{cases}
\end{align*}
reaches its maximum over the boundary $ \| v\|=1$. 
For all $v\in \ngradnorm{\|\cdot\|}{L(\theta)}$ such that $\| v\|=1$, we know 
\begin{align*}
    \| v \|_2 \geq c_{\| \cdot \|} \| v\| = c_{\| \cdot \|}. 
\end{align*}
We conclude the proof by taking the $\inf$ over $v$.
\end{proof}

\subsubsection{Proofs of \Cref{thm-ngrad-entry-time-upper-1,thm-gradnorm-entry-time-upper-1}}
\label{Proof-entry-time-1}
\paragraph{Common pre-entry estimates}
We use the following elementary estimates in the entry-time proofs.
\begin{itemize}
\item The constant \(\lambda_{\max}\) is any constant satisfying
\Cref{assumption-l-smooth}; for a \(C^2\) loss, the Lipschitz-gradient part is
ensured by
\[
\lambda_{\max}\geq
\sup_{\theta\in\mathbb{R}^d}
\sigma_{\max}
\bigl(\nabla^2L(\theta)\bigr).
\]
The parameter \(\psi\) is the granularity of the admissible minority class
\(\Kmin\), and \(D\) is the offset scale in the bound
\(\|\nabla L_0(\theta)\|_2\leq\psi(\|\theta\|_2+D)\).
\item We know from \Cref{cor-non-emptyness-robust-zones} that $\|\nabla L(\theta) \| < \psi D \implies \theta \in \Zmin$. By the continuity of $\theta \xrightarrow{} \| \nabla L(\theta)\|$, we deduce that if $\nabla L(\thetaInit) \neq 0 $, there is no critical point before the entry time. As consequence, from \Cref{lem-lower-bound-m}, $m_{\| \cdot \|}$ is greater than $c_{\| \cdot\|}$ before the entry time as well. 
\item Recall that 
\[
m_{\|\cdot\|}(\theta):=
\operatorname{dist}\bigl(0,\ngradnorm{\|\cdot\|}{L(\theta)}\bigr)
=
\inf_{v\in \ngradnorm{\|\cdot\|}{L(\theta)}}\|v\|_2 .
\]
For the normalized gradient $\ngradnorm{\|\cdot\|}{L(\theta)}$ dynamics, the cost encoding the minority aware zone is 
\[
\Qrob_{\|\cdot\|}(\theta)
=
\|\mathrm{d}L(\theta)\|_{\ast}
-
\psi m_{\|\cdot\|}(\theta)(\|\theta\|_2+D).
\]

For the normal gradient  \(\gradnorm{\|\cdot\|}{L(\theta)}\), the cost is 
\[
\Qrob(\theta)
=
\|\mathrm{d}L(\theta)\|_{\ast}^2
-
\psi(\|\theta\|_2+D) m_{\|\cdot\|}(\theta)
\|\mathrm{d}L(\theta)\|_{\ast}.
\]

Consequently, before entry,
\begin{equation}\label{preentrybound-1}
\|\mathrm{d}L(\theta(t))\|_{\ast}
\geq
\psi \, m_{\|\cdot\|}(\theta(t))(\|\theta(t)\|_2+D)
\geq \psi D m_{\|\cdot\|}(\theta(t)).
\end{equation}
\end{itemize}
We note that from \Cref{assumption-l-smooth}, one has the following lower bound for all \(\theta\)
\begin{equation}\label{eq:Lthetabound}
\|\theta\|_2+D
\geq
\sqrt{\frac{L(\theta)-L_\star}{\lambda_{\max}}}.
\end{equation}
Consequently, before entry,
\begin{equation}\label{preentrybound-2}
\|\mathrm{d}L(\theta(t))\|_{\ast}
\geq
\psi \, m_{\|\cdot\|}(\theta(t))(\|\theta(t)\|_2+D)
\geq
\frac{\psi m_{\|\cdot\|}(\theta(t))}{\sqrt{\lambda_{\max}}}
\sqrt{L(\theta(t))-L_\star}.
\end{equation}

\textbf{Two differential inequalities}

Let
\[
G(t):=L(\theta(t))-L_\star
\quad (t\geq 0),
\qquad
C:=
\frac{L(\thetaInit)-L_\star}
{\lambda_{\max}D^2}.
\]
Recall from \Eqref{eq:maverage} that
\[
\bar m(t):=
\frac{1}{t}
\int_0^t
m_{\|\cdot\|}(\theta(s))\,ds
\quad (t>0).
\]
Since there are no critical points before the entry time, one has
\[
G(t)>0
\quad\text{for all }t<\timeentry(\theta).
\]
Indeed, \(G(t)=0\) would imply that \(\theta(t)\) is a global minimizer
of \(L\), and hence that \(\nabla L(\theta(t))=0\).

\textbf{a) Normalized norm-gradient.}

Assume
\[
\dot\theta(t)\in
-\ngradnorm{\|\cdot\|}{L(\theta(t))}
\]
and that
\[
\nabla L(\thetaInit)\neq 0.
\]
Set
\[
T:=\timeentry(\theta).
\]
By the common pre-entry estimates above, there is no critical point
before the entry time. Consequently, by
\Cref{lem-lower-bound-m},
\[
m_{\|\cdot\|}(\theta(t))
\geq
c_{\|\cdot\|}
>0
\quad\text{for all }t<T.
\]

For almost every \(t<T\), the normalized norm-gradient dynamics gives
\[
-G'(t)
=
\|\mathrm{d}L(\theta(t))\|_{\ast}.
\]
Hence, by \Eqref{preentrybound-1},
\[
-G'(t)
\geq
\psi D
m_{\|\cdot\|}(\theta(t)).
\]
Therefore, for every \(t\) such that \(0<t<T\),
\[
\begin{aligned}
G(0)-G(t)
&\geq
\psi D
\int_0^t
m_{\|\cdot\|}(\theta(s))\,ds
\\
&=
\psi D t\bar m(t).
\end{aligned}
\]
Since \(G(t)\geq 0\), it follows that
\begin{equation}
\label{eq:tbarmt}
t\bar m(t)
\leq
\frac{G(0)}{\psi D}.
\end{equation}

Moreover, from \Cref{lem-lower-bound-m},
\[
\bar m(t)
=
\frac{1}{t}
\int_0^t
m_{\|\cdot\|}(\theta(s))\,ds
\geq
c_{\|\cdot\|}.
\]
The entry time \(T\) must therefore be finite. Indeed, if
\(T=+\infty\), then
\[
t c_{\|\cdot\|}
\leq
\frac{G(0)}{\psi D}
\]
would hold for every \(t>0\), which is impossible.

If \(T=0\), the theorem holds. Suppose now
that \(T>0\). Letting \(t\) increase to \(T\) in the preceding
inequality, Equation \ref{eq:tbarmt}, gives
\[
\int_0^T
m_{\|\cdot\|}(\theta(s))\,ds
\leq
\frac{G(0)}{\psi D}.
\]
By the definition of \(\bar M_\theta\) in \Eqref{eq:barMtheta},
\[
T\bar M_\theta
\leq
\frac{G(0)}{\psi D}.
\]
Since
\[
G(0)=L(\thetaInit)-L_\star,
\]
we obtain
\[
\timeentry(\theta)
\leq
\frac{L(\thetaInit)-L_\star}
{\psi D\bar M_\theta}.
\]

Now suppose that \Cref{assumption-l-smooth} also holds. By
\Eqref{preentrybound-2}, for almost every \(t<T\),
\[
-G'(t)
=
\|\mathrm{d}L(\theta(t))\|_{\ast}
\geq
\frac{\psi
m_{\|\cdot\|}(\theta(t))}
{\sqrt{\lambda_{\max}}}
\sqrt{G(t)}.
\]
Recall that
\[
G(t)>0
\quad\text{for all }t<T.
\]
Thus, for almost every \(t<T\),
\[
\begin{aligned}
\frac{d}{dt}\sqrt{G(t)}
&=
\frac{G'(t)}
{2\sqrt{G(t)}}
\\
&\leq
-\frac{\psi}
{2\sqrt{\lambda_{\max}}}
m_{\|\cdot\|}(\theta(t)).
\end{aligned}
\]
Integrating over \([0,t]\), for every \(t\) such that \(0<t<T\),
gives
\[
\sqrt{G(t)}
\leq
\sqrt{G(0)}
-
\frac{\psi}
{2\sqrt{\lambda_{\max}}}
\int_0^t
m_{\|\cdot\|}(\theta(s))\,ds.
\]
Since \(G(t)\geq 0\), we obtain
\[
\int_0^t
m_{\|\cdot\|}(\theta(s))\,ds
\leq
\frac{2\sqrt{\lambda_{\max}}}{\psi}
\sqrt{G(0)}.
\]
Letting \(t\) increase to \(T\) gives
\[
T\bar M_\theta
\leq
\frac{2\sqrt{\lambda_{\max}}}{\psi}
\sqrt{G(0)}.
\]
Consequently,
\[
\timeentry(\theta)
\leq
\frac{2\sqrt{\lambda_{\max}}}
{\psi\bar M_\theta}
\sqrt{L(\thetaInit)-L_\star}.
\]

\textbf{b) Norm-gradient without normalization.}
Assume
\[
\dot\theta(t)\in-\gradnorm{\|\cdot\|}{L(\theta(t))}
\]
and that $\nabla L(\thetaInit) \neq 0$. We get for $t < \timeentry(\theta)$
\[
\|\mathrm{d}L(\theta(t))\|_{\ast}^2
\geq
\psi^2 D^2  m_{\|\cdot\|}(\theta(t))^2. 
\]
For a.e. \(t<\timeentry(\theta)\),  \Cref{eq-descent-grad} gives
\[
-G'(t)
=
\|\mathrm{d}L(\theta(t))\|_{\ast}^2
\geq
\psi^2 D^2
m_{\|\cdot\|}(\theta(t))^2.
\]
By Jensen's inequality,
\[
\int_0^t
m_{\|\cdot\|}(\theta(s))^2\,ds
\geq
t\bar m(t)^2 \geq t c_{\| \cdot\|}^2 .
\]
Since $G(t) \geq 0$, we obtain that $\timeentry$ is finite and
\begin{equation}
\label{eq-gradient-without-normalization-1}
\bar M_\theta^2  \timeentry \leq \frac{G(0)}{\psi^2 D^2 }.
\end{equation}

Now suppose that \Cref{assumption-l-smooth} holds. Set
\[
T:=\timeentry(\theta)
\qquad\text{and}\qquad
C:=
\frac{G(0)}{\lambda_{\max}D^2}
=
\frac{L(\thetaInit)-L_\star}
{\lambda_{\max}D^2}.
\]
Recall that \(C>0\), since
\(\nabla L(\thetaInit)\neq 0\).

Let \(\beta\in(0,2]\). For almost every \(t<T\), using
Equations 
\ref{preentrybound-1} and \ref{preentrybound-2}, we obtain
\[
\begin{aligned}
-G'(t)
&=
\|\mathrm{d}L(\theta(t))\|_{\ast}^2
\\
&=
\|\mathrm{d}L(\theta(t))\|_{\ast}^{2-\beta}
\|\mathrm{d}L(\theta(t))\|_{\ast}^{\beta}
\\
&\geq
\left(
\frac{\psi m_{\|\cdot\|}(\theta(t))}
{\sqrt{\lambda_{\max}}}
\sqrt{G(t)}
\right)^{2-\beta}
\left(
\psi Dm_{\|\cdot\|}(\theta(t))
\right)^{\beta}
\\
&=
\frac{\psi^2}{\lambda_{\max}}
\left(\lambda_{\max}D^2\right)^{\beta/2}
m_{\|\cdot\|}(\theta(t))^2
G(t)^{1-\beta/2}.
\end{aligned}
\]
Since \(G(t)>0\) for every \(t<T\), it follows that, for almost every
\(t<T\),
\[
\frac{d}{dt}G(t)^{\beta/2}
=
\frac{\beta}{2}
G(t)^{\beta/2-1}G'(t)
\leq
-\frac{\beta\psi^2}{2\lambda_{\max}}
\left(\lambda_{\max}D^2\right)^{\beta/2}
m_{\|\cdot\|}(\theta(t))^2.
\]
Therefore, for every \(t<T\),
\[
\begin{aligned}
G(t)^{\beta/2}
&\leq
G(0)^{\beta/2}
\\
&\quad
-\frac{\beta\psi^2}{2\lambda_{\max}}
\left(\lambda_{\max}D^2\right)^{\beta/2}
\int_0^t
m_{\|\cdot\|}(\theta(s))^2\,ds.
\end{aligned}
\]
Since \(G(t)\geq0\), this gives
\[
\int_0^t
m_{\|\cdot\|}(\theta(s))^2\,ds
\leq
\frac{2\lambda_{\max}}{\beta\psi^2}
C^{\beta/2}.
\]
Letting \(t\) increase to \(T\), and using Jensen's inequality, we obtain
\[
\begin{aligned}
T\bar M_\theta^2
&=
\frac{1}{T}
\left(
\int_0^T
m_{\|\cdot\|}(\theta(s))\,ds
\right)^2
\\
&\leq
\int_0^T
m_{\|\cdot\|}(\theta(s))^2\,ds
\\
&\leq
\frac{2\lambda_{\max}}{\beta\psi^2}
C^{\beta/2}.
\end{aligned}
\]
Consequently, for every \(\beta\in(0,2]\),
\[
T
\leq
\frac{\lambda_{\max}}
{\psi^2\bar M_\theta^2}
\frac{2C^{\beta/2}}{\beta}.
\]

It remains to minimize the function
\[
f_C(\beta):=
\frac{2C^{\beta/2}}{\beta},
\qquad
\beta\in(0,2].
\]
Its derivative is
\[
f_C'(\beta)
=
\frac{2C^{\beta/2}}{\beta^2}
\left(
\frac{\beta}{2}\log C-1
\right).
\]
If \(0<C\leq e\), the minimum of \(f_C\) over \((0,2]\) is attained at
\(\beta=2\), and
\[
\inf_{\beta\in(0,2]}f_C(\beta)
=
f_C(2)
=
C.
\]
If \(C>e\), the minimum is attained at
\[
\beta_\star=
\frac{2}{\log C}\in(0,2),
\]
and
\[
\inf_{\beta\in(0,2]}f_C(\beta)
=
f_C(\beta_\star)
=
e\log C.
\]
We have thus proved the piecewise bound
\[
\timeentry(\theta)
\leq
\frac{\lambda_{\max}}
{\psi^2\bar M_\theta^2}
\begin{cases}
C,
&
0<C\leq e,
\\[1mm]
e\log C,
&
C>e.
\end{cases}
\]

\subsubsection{Clipped gradient descent}
\label{sec:appendix-entry-time-clipped}
\begin{theorem}[Upper bound of entry time for clipped gradient descent]
\label{thm-clip-entry-time-upper}
Suppose that \Cref{ass-cost,ass-VL} hold and consider the clipped gradient flow
\begin{align*}
\theta(0) &= \thetaInit \in \mathbb{R}^d,\\
\forall t \ge 0, \quad \dot{\theta}(t)
&=
- \Clip_{\tau_{\rm clip}}(\nabla L(\theta(t))),
\end{align*}
with threshold \(\tau_{\rm clip}>0\).
Assume that $\nabla L (\thetaInit) \neq 0$.
Then, 
\begin{equation} 
\label{eq:claim:clipped:one}
\timeentry(\theta)
\leq
\frac{L(\thetaInit)-L_\star}
{\min\{\psi^2D^2,\tau_{\rm clip}\psi D\}}.
\end{equation}
If \Cref{assumption-l-smooth} holds, we obtain that 
\begin{align}
\label{eq:claim:clipped:two}
\timeentry(\theta)
\leq
\frac{2\sqrt{\lambda_{\max}}}{\psi \min \{\tau_{\rm clip},\psi D \}}
\sqrt{L(\thetaInit)-L_\star}.
\end{align}
\end{theorem}

\begin{proof}
Set \(E(t):=L(\theta(t))-L_\star\). Before entry,
\Eqref{eq:robust-zone-clip} gives
\[
\|\nabla L(\theta(t))\|_2\geq \psi(\|\theta(t)\|_2+D)\geq\psi D .
\]
As in the proof of \Cref{thm-ngrad-entry-time-upper-1,thm-gradnorm-entry-time-upper-1},
the chain rule gives, for a.e. pre-entry time,
\[
E'(t)
=
-\min\{\|\nabla L(\theta(t))\|_2^2,
\tau_{\rm clip}\|\nabla L(\theta(t))\|_2\}
\leq
-\min\{\psi^2D^2,\tau_{\rm clip}\psi D\}.
\]
Integrating up to any \(s<\timeentry(\theta)\) and using \(E(s)\geq0\) gives
the claim in \Eqref{eq:claim:clipped:one}.

Now suppose that \Cref{assumption-l-smooth} holds. We have for $t< \timeentry$, from \Eqref{preentrybound-2} and  with
\(m_{\|\cdot\|}\equiv1\),
\begin{align*}
E'(t) \leq - \min \{ \tau_{\rm clip},\psi D \} \| \nabla L(\theta(t))\|_2 \leq - \min \{ \tau_{\rm clip},\psi D \} \frac{\psi}{\sqrt{\lambda_{\max}}} \sqrt{E(t)}.
\end{align*}
Similarly to \Cref{thm-ngrad-entry-time-upper-1}, we deduce that
\[
2\frac{d}{dt}\sqrt{L(\theta(t))-L_\star}
\leq
-\frac{\psi \min \{ \tau_{\rm clip},\psi D \}}{\sqrt{\lambda_{\max}}}
\]
for a.e. pre-entry time. Integration gives the stated bound in \Eqref{eq:claim:clipped:two}.
\end{proof}

\subsubsection{Newton's descent}
\label{sec:appendix-entry-time-newton}

\begin{theorem}[Upper bound of entry time for Newton method]
\label{thm-newton-entry-time-upper}
Consider a solution $\theta(t)$ to the following dynamical system,
\begin{align*}
\theta(0) &= \thetaInit \in \mathbb{R}^d \\
\forall t \ge 0, \quad \dot{\theta}(t) & = -\left(\nabla^2 L(\theta(t))\right)^{-1} \nabla L(\theta(t)),
\end{align*}
such that $\nabla^2L(\theta(t))$ is positive definite along the trajectory.
Suppose that we have at the initialization,
\[ \| \nabla L(\thetaInit) \|_2 > \psi ( \|\thetaInit\|_2 +  D). \]
Then we know that 
\begin{align*}
 \timeentry(\theta)  \leq \log\left( \frac{\| \nabla L(\thetaInit) \|_2 }{\psi D }\right). 
\end{align*}
\end{theorem}

\begin{proof}
Along the Newton flow,
\[
\frac{d}{dt}\nabla L(\theta(t))
=
\nabla^2L(\theta(t))\dot\theta(t)
=
-\nabla L(\theta(t)),
\]
hence \(\|\nabla L(\theta(t))\|_2=\|\nabla L(\thetaInit)\|_2e^{-t}\).
By \Eqref{eq:calQ:newton} and Cauchy--Schwarz,
\[
\|\nabla L(\theta(t))\|_2<\psi D
\quad\Longrightarrow\quad
\Qrob(\theta(t))<0.
\]
Thus the first time at which
\(\|\nabla L(\thetaInit)\|_2e^{-t}<\psi D\) is an upper bound for
\(\timeentry(\theta)\), which gives the result.
\end{proof}

\subsubsection{Applications to standard descents}
\label{sec:appendix-entry-time-applications}
We apply \Cref{thm-ngrad-entry-time-upper-1,thm-gradnorm-entry-time-upper-1}.

\paragraph{Gradient descent.}
For \(\|\cdot\|=\|\cdot\|_2\), before entry
\[
\ngradnorm{\|\cdot\|_2}{L(\theta)}
=
\left\{\frac{\nabla L(\theta)}{\|\nabla L(\theta)\|_2}\right\},
\qquad
m_{\|\cdot\|_2}(\theta)=1.
\]
Thus \(\bar m(s)=1\), and \Cref{thm-gradnorm-entry-time-upper-1} gives the stated bound.

\paragraph{Sign descent.}
For \(\|\cdot\|=\|\cdot\|_\infty\),
\begin{equation} \label{eq:m:SD}
m_{\|\cdot\|_\infty}(\theta)
=
\operatorname{dist}\bigl(0,\sign(\nabla L(\theta))\bigr)
=
\sqrt{|\supp(\nabla L(\theta))|}.
\end{equation}
Applying \Cref{thm-ngrad-entry-time-upper-1} gives the stated bound.

\paragraph{Spectral descent.}
Let \(G=\nabla L(\Gamma)\) have rank \(r\) and compact SVD \(G=U\Sigma V^\top\).
For \(\|\cdot\|_{\rm sp}\),
\begin{equation} \label{eq:rho:grad:spectral}
\ngradnorm{\|\cdot\|_{\rm sp}}{L(\Gamma)}
=
\{UV^\top+W:\ U^\top W=0,\ WV=0,\ \|W\|_{\rm sp}\leq1\},
\end{equation}
so
\begin{equation} \label{eq:bar:M:spec}
m_{\|\cdot\|_{\rm sp}}(\Gamma)
=
\operatorname{dist}\bigl(0,\ngradnorm{\|\cdot\|_{\rm sp}}{G}\bigr)
=
\sqrt{r}.
\end{equation}
Applying \Cref{thm-ngrad-entry-time-upper-1} gives the stated bound.

\subsection{Lower bounds}
\label{sec-lower-bounds}

We will build a one-dimensional function that will be useful for the lower bounds in this section. Consider $K \geq \psi $ where $K$ is a parameter of the function, that quantifies its second derivative on the positive part of its input space. For all $\delta >0$:
\begin{equation}
\label{eq-definition-gradient-descent-lower-bound}
\begin{aligned}
g_{\delta}(-\delta) & := 0 \\
g_{\delta}'(\theta)  & :=  \begin{cases}  K \theta + \psi D  & \text{for } \theta \ge 0 \\ P(\theta) & \text{for } \theta \in [-\delta, 0) \\ 0 & \text{for } \theta \le -\delta \end{cases}
\end{aligned}
\end{equation}
where $P$ is an interpolation polynomial of degree at most $3$ defined in the following way:
\begin{align*}
P(\theta) & := A(\theta+\delta)^3 + B(\theta+\delta)^2 \\
A & := \frac{K \delta - 2 \psi D}{\delta^3}\\
B& := \frac{3\psi D  - K \delta}{\delta^2}.
\end{align*}

\begin{lemma}
\label{lemma:properties:g}
For all $\delta >0$, $g_{\delta} \in \mathcal{C}^2(\mathbb{R},\mathbb{R})$, $g_{\delta}(0) = \frac{\delta \psi D }{2} - \frac{K \delta^2}{12} $ and 
\[ \forall 0<x<\delta, \quad |g_{\delta}(0) - g_{\delta}(-x)| \leq  K \delta^2 +  2 \delta \psi D. \]
By looking at the sign of $g_{\delta}'$, we know
\[g_{\delta}(0) - \min g_{\delta } \leq  K \delta^2 +  2 \delta \psi D.\]
\end{lemma}
\begin{proof}
The function
$g_{\delta}$ is $\mathcal{C}^2$ in each segment separately. We verify that the derivative $g'_{\delta}$ is continuous at the borders: 
\begin{align*}
P(0) & = A\delta^3 + B \delta^2 = \psi D  \\
P(-\delta) & = 0,
\end{align*}
and also $g'_{\delta}$ is continuously differentiable at the borders
 \begin{align*} 
P'(0) & = 3 A\delta^2 + 2 B \delta =  \frac{1}{\delta} \left( 3 K \delta - 6 \psi D  + 6\psi D  - 2 K\delta \right)  =K\\
P'(-\delta) & = 0.
\end{align*}
Now we can calculate the value at $0$:
\begin{align*}
g_{\delta }(0) = \int_{-\delta}^{0} P(\theta) d\theta & = \Big[ A\frac{(\theta+\delta)^4}{4} +  B\frac{(\theta+\delta)^3}{3} \Big]_{-\delta}^{0} \\
& = \frac{A\delta^4}{4} + \frac{B \delta^3}{3} \\
& = \frac{K \delta^2- 2\delta \psi D }{4} + \frac{3\delta \psi D - K \delta^2}{3} \\
& = \frac{\delta \psi D }{2} - \frac{K \delta^2}{12}.
\end{align*}
Last but not least, for $0<x<\delta$,
\begin{align*}
g_{\delta }(0)  - g_{\delta }(-x)  = \int^{0}_{-x} P(\theta) d\theta & = \Big[ A\frac{(\theta+\delta)^4}{4} +  B\frac{(\theta+\delta)^3}{3} \Big]_{-x}^{0} \\
& =
\frac{\delta \psi D }{2} - \frac{K \delta^2}{12}
-\frac{A (\delta-x)^4}{4} - \frac{B (\delta-x)^3}{3}. 
\end{align*}
By taking the absolute values:
\begin{align*}
| g_{\delta }(0) - g_{\delta }(-x) | & \leq \frac{\delta}{4} |K\delta - 2 \psi D  | + \frac{\delta}{3} |3 \psi D  - K\delta| + \frac{\delta \psi D }{2} + \frac{K \delta^2}{12}  \\
& \leq (\frac{1}{4} + \frac{1}{3} + \frac{1}{12}) K \delta^2 + \delta \psi D  (\frac{1}{2} + 1 + \frac{1}{2}). 
\end{align*}
\end{proof}

\subsubsection{Lower bound for Gradient Descent}
For gradient descent, we stay in dimension one with $g_{\delta }$ for the lower bound.
The strategy is to lower-bound the entry time by the initial suboptimality gap divided by \(3 \psi^2D^2\), up to an error that vanishes as \(\delta\to0\) (\Cref{remark:LB:GD}).

\begin{theorem}
\label{th:LB:min:awa}
Consider $0 < \thetaInit \leq \frac{\psi}{K}D$. For any $\delta >0$, we follow the Gradient descent dynamics with $g_{\delta}$ defined in \Cref{eq-definition-gradient-descent-lower-bound}. Then the entry time is lower bounded as
\begin{equation} \label{eq:LB:GD}
\timeentry \geq \frac{g_{\delta}(\thetaInit) - \min g_{\delta}}{3 \psi^2 D^2}  - \frac{K \delta^2}{3 \psi^2 D^2} - \frac{  2 \delta  }{3 \psi D}.
\end{equation}
\end{theorem}
\begin{proof}
We recall from \Cref{table:zoology:minority:zone} that the entry time condition for Gradient descent in dimension one is 
\[ 0 < |g'(\theta)| < \psi (|\theta| + D).\]
Note that we initialize at the positive part $\theta_{init} > 0$. We also notice that for $\theta > 0$,
\begin{align*}
| g'_{\delta}(\theta) |  =  K\theta +  \psi D \geq \psi (\theta + D). 
\end{align*}
Therefore, the entry time is at least the travelling time $T_0$ from $\thetaInit$ to $0$. Now we know for $0 < t < T_0$:
\begin{align*}
\theta'(t) & = - g'_{\delta}(\theta(t)) = - K \theta(t) - \psi D  \\
\theta(t) & = \mathrm{Const} \exp( - K t) - \frac{\psi D }{K}
\end{align*}
with $\mathrm{Const}$ a constant that depends on the initialization $\theta(t=0)=\thetaInit$, 
that is $\mathrm{Const} = \thetaInit +  \frac{\psi D}{K}$.
Next, we have
\begin{align*}
 \theta(t)  & = \left(\thetaInit +   \frac{\psi D}{K} \right)  \exp( - K t) - \frac{\psi D}{K}  \\
\implies  \quad  0 & = \left(\thetaInit +   \frac{\psi D}{K} \right)  \exp( - K T_0) - \frac{\psi D}{K}  \\
\implies  \quad  T_0 & = \frac{1}{K} \log \left( \frac{\thetaInit + \frac{\psi D}{K} }{\frac{\psi D}{K} }\right) = \frac{1}{K} \log \left( 1 + \frac{K \thetaInit}{\psi D} \right).
\end{align*}
The last step is to introduce the suboptimality gap:
\begin{align*}
g_{\delta}(\thetaInit) - g_{\delta}(0) &  = \int_{0}^{\thetaInit} \left( K \theta + \psi D \right) \, d\theta = \frac{K}{2}\theta_{init}^2 + \psi D\theta_{init} \\
\frac{\psi^2 D^2}{2K}+ g_{\delta}(\thetaInit) - g_{\delta}(0) & = \frac{1}{2} (\sqrt{K}\thetaInit +  \frac{\psi D}{\sqrt{K}})^2 \\
1+ \frac{2K(g_{\delta}(\thetaInit) - g_{\delta}(0))}{\psi^2 D^2}  & =  \Big( \frac{K\thetaInit}{\psi D}  + 1\Big)^2. 
\end{align*}
Now we can rearrange the terms:
\begin{align*}
T_0 = \frac{1}{K} \log \Big( 1 + \frac{K\thetaInit}{\psi D}\Big) = \frac{1}{2 K} \log \left( 1 +  \frac{2K(g_{\delta}(\thetaInit) - g_{\delta}(0))}{\psi^2 D^2} \right).
\end{align*}
Since we supposed that $\thetaInit \leq \frac{\psi}{K}D$, we know that 
\begin{align*}
2K(g_{\delta}(\thetaInit) - g_{\delta}(0)) & = K^2 \thetaInit^2 + 2K \psi \thetaInit D  \leq 3 \psi^2 D^2.
\end{align*}
Since $\forall x\in[0,3], \ln (1+x) \geq x/3$, we get 
\begin{align*}
t_{entry} \geq T_0 & \geq  \frac{1}{2K} \frac{2 K(g_{\delta}(\thetaInit) - g_{\delta}(0))}{3 \psi^2 D^2}  \\
& = \frac{g_{\delta}(\thetaInit) - \min g_{\delta}}{3\psi^2 D^2} + \frac{\min g_{\delta} -  g_{\delta}(0) }{3 \psi^2 D^2} \\
& \geq \frac{g_{\delta}(\thetaInit) - \min g_{\delta}}{3 \psi^2 D^2} - \frac{K \delta^2 +  2 \delta \psi D  }{3\psi^2 D^2}
\\
& = \frac{g_{\delta}(\thetaInit) - \min g_{\delta}}{3 \psi^2 D^2} - \frac{K \delta^2   }{3\psi^2 D^2}
- \frac{ 2\delta   }{3\psi D},
\end{align*}
where we used \Cref{lemma:properties:g} in the last inequality.
\end{proof}

\begin{remark}
\label{remark:LB:GD}
In \Cref{th:LB:min:awa}, for fixed  $\thetaInit$, $K$, $\psi$ and $D$, we can let $\delta >0$ go to zero. In this case, in the right-hand side of \Eqref{eq:LB:GD}, the quantity $- \frac{K \delta^2}{3 \psi^2 D^2} - \frac{  2 \delta  }{3 \psi D}$ vanishes, while the quantity $g_{\delta}(\thetaInit) - \min g_{\delta}$ converges to $K \thetaInit^2/2 + \psi D \thetaInit >0$. Hence, for $\delta$ small enough, the right-hand side of \Eqref{eq:LB:GD} is at least $\frac{g_{\delta}(\thetaInit) - \min g_{\delta}}{6 \psi^2 D^2}$.
Note also that the trajectory in \Cref{th:LB:min:awa} encounters a global loss function that has second derivative $K$, which is arbitrarily larger than the parameter $\psi$ defining the minority loss class in \Eqref{def-K}. A similar remark also applies to \Cref{th:LB:min:awa:sign,th:LB:min:awa:spec}.  
\end{remark}

\begin{remark}
Consider the lower bound $\timeentryLB$ as the right-hand side of \Eqref{eq:LB:GD} in \Cref{th:LB:min:awa}, so that $\timeentry \ge \timeentryLB$. Recall that $\timeentry$ is the entry time to the minority aware zone $\Zmin
=
\bigcup_{L_0\in\Kmin}\Zmaj(L_0)$.
For any decomposition $L = L_0 + L_1$,
consider the entry time to the majority adverse zone $\Zmaj(L_0)$, denoted $\timeentryMaj(L_0)$. Then, because  $\Zmaj(L_0) \subseteq \Zmin$, we have $\timeentryMaj(L_0) \ge \timeentry \ge \timeentryLB$. In words, $\timeentryLB$ is a lower bound not only on the minority aware zone, that is oblivious to the choice of the decomposition $L = L_0 + L_1$, but also on the majority adverse zone, for any choice of the decomposition $L = L_0 + L_1$. This remark also applies to \Cref{th:LB:min:awa:sign,th:LB:min:awa:spec}. 
\end{remark}

\subsubsection{Lower bound for Sign Descent}
Now we will make a small modification to $g_\delta$ by substituting $D \xleftarrow{} D/\sqrt{k}$ with $k \geq 1$. We define for all $\delta >0$:
\begin{equation}
\begin{aligned}
g_{\delta,k}(-\delta) & := 0 \\
g_{\delta,k}'(\theta)  & :=  \begin{cases}   K \theta + \frac{\psi D}{\sqrt{k}}  & \text{for } \theta \ge 0 \\ P(\theta) & \text{for } \theta \in [-\delta, 0) \\ 0 & \text{for } \theta \le -\delta \end{cases}
\end{aligned}
\end{equation}
where $P$ is an interpolation polynomial of degree at most $3$ defined in the following way:
\begin{align*}
P(\theta) & := A(\theta+\delta)^3 + B(\theta+\delta)^2 \\
A & := \frac{K\delta - 2 \psi (D/\sqrt{k})}{\delta^3}\\
B& := \frac{3\psi (D/\sqrt{k})  - K \delta}{\delta^2}.
\end{align*}
We know that the following lemma holds for the same reason as before.
\begin{lemma}
For all $\delta >0$, $g_{\delta,k} \in \mathcal{C}^2(\mathbb{R},\mathbb{R})$, $g_{\delta,k}(0) = \frac{\delta \psi D }{2\sqrt{k}} - \frac{K \delta^2}{12} $ and 
\[ \forall 0<x<\delta, \quad |g_{\delta,k}(0) - g_{\delta,k}(-x)| \leq  K \delta^2 +  \frac{2 \delta \psi D}{\sqrt{k}}.\]
\end{lemma}

Now we consider a class of diagonal functions. Let $1 \leq k \leq d$, and let
\begin{equation}
\label{eq-def-sign-descent-lower-bound}
\forall \theta \in \mathbb{R}^d, \quad L_{\delta, k}(\theta) := \sum_{i=1}^{k} g_{\delta,k}(\theta_i).
\end{equation}

In the next theorem, $\textbf{1}_k$ is the $d$-dimensional vector with first $k$ components equal to one, and the remaining ones equal to zero.

\begin{theorem}
\label{th:LB:min:awa:sign}
Consider $\delta >0$, $\thetaInit = R \textbf{1}_k$ with $R > 0$ and assume that $\| \thetaInit\|_2 \leq \frac{\psi D}{K}$ and that we follow the Sign Descent dynamics of $L_{\delta,k}$ defined in \Cref{{eq-def-sign-descent-lower-bound}}. Then, there exists an absolutely continuous solution of the sign-descent inclusion for which the entry time is lower bounded as
 \begin{align*}
\timeentry \geq  \frac{2}{3 \bar M_{\rm SD}} \frac{L_{\delta,k}(\thetaInit) - \min L_{\delta,k}}{\psi D} - \frac{K \bar M_{\rm SD} \delta^2 }{ \psi D } - 2 \delta.
 \end{align*}
\end{theorem}

\begin{proof}
First, by the symmetry of the system, there exists a solution that always stays on the diagonal:
\begin{align*}
\forall t \geq 0, \quad \theta(t) = r(t) \textbf{1}_k
\end{align*}
where $r : \mathbb{R}_{+} \xrightarrow{} \mathbb{R}$.  We notice that for $1 \leq i \leq k$, if $\theta_i(t) = r(t) \geq 0$,
\begin{align*}
\nabla_i L_{\delta, k} (\theta(t)) & = K r(t) + \frac{\psi D}{\sqrt{k}}. 
\end{align*}
Therefore for $t \geq 0$
\begin{align*}
r(t)  >  0 \implies & r'(t) = -\sign(K r(t) + \frac{\psi D}{\sqrt{k}} )  = -1. 
\end{align*}
Therefore, we can define the travel time $T_0$ for $r(t)$ from $R$ to $0$ and we have,
\[ \forall t \in [0,T_{0}), \quad r(t) = R - t, \quad |\supp(\nabla L_{\delta,k}(\theta(t)))| = k.\]
Now we calculate the entry time. We recall from Table \ref{table:zoology:minority:zone} that the entry time condition for Sign Descent is
\[ \| \nabla L(\theta) \|_1 < \psi \sqrt{|\supp(\nabla L(\theta))|}(\| \theta \|_2 + D ). \]
In our diagonal setting, for $0 \leq t < T_0$,
\begin{align*}
\| \nabla L(\theta(t))\|_1 & =  k |  K r(t) + \frac{\psi D}{\sqrt{k}} | \\
& =  | k K r(t) +  \sqrt{k} \psi  D |\\
\psi \sqrt{|\supp(\nabla L(\theta(t)))|}(\| \theta(t)\|_2 + D ) & = \psi \sqrt{k} (\sqrt{k}|r(t)| + D) \\
 & = \psi k |r(t)| + \psi \sqrt{k} D.
\end{align*}
We notice that 
\[ r(t) \geq  0 \implies \| \nabla L(\theta(t))\|_1  \geq \psi \sqrt{|\supp(\nabla L(\theta(t)))|}(\| \theta(t)\|_2 + D ). \]
Therefore, we deduce that $R=T_0 \leq \timeentry$.

We now show that $\timeentry=R$. Since
$g_{\delta,k}'(0)=\psi D/\sqrt{k}>0$ and
$g_{\delta,k}''(0)=K>0$, there exists
$\varepsilon\in(0,\delta)$ such that
\[
0<g_{\delta,k}'(r)<\frac{\psi D}{\sqrt{k}}
\qquad\text{for every }r\in(-\varepsilon,0).
\]
Consequently, the chosen diagonal solution continues
to satisfy $r(t)=R-t$ for $R<t<R+\varepsilon$. For these times,
$|\supp(\nabla L_{\delta,k}(\theta(t)))|=k$ and
\[
\begin{aligned}
\|\nabla L_{\delta,k}(\theta(t))\|_1
&=k\,g_{\delta,k}'(r(t))\\
&<\psi D\sqrt{k}\\
&\le \psi\sqrt{k}\bigl(\|\theta(t)\|_2+D\bigr).
\end{aligned}
\]
Thus the trajectory lies in the minority aware zone
for every $t\in(R,R+\varepsilon)$, which gives $\timeentry\leq R$.
Together with $\timeentry\geq R$, this proves $\timeentry=R$.

Now, we introduce the suboptimality gap. 
\begin{align*}
L_{\delta,k}(\thetaInit) - L_{\delta,k}(0) & = k \left( g_{\delta,k}(R) - g_{\delta,k}(0) \right) \\
& = k \int_{0}^{R} g_{\delta,k}'(s) ds \\
& = k \Big( \frac{K R^2}{2} + \frac{\psi D R}{\sqrt{k}}\Big) \leq \frac{3\psi D R \sqrt{k}}{2}
\end{align*}
since $\| \thetaInit\|_2= R\sqrt{k} \leq \frac{\psi D}{K} $. Therefore
\begin{align*}
    \timeentry \geq T_0 = R \geq \frac{2}{3 \sqrt{k}} \frac{L_{\delta,k}(\thetaInit) - L_{\delta,k}(0)}{\psi D}.
\end{align*}
By construction, we also know that $L_{\delta,k}(0) - \min L_{\delta,k} \leq k K\delta^2 + 2 \delta \sqrt{k} \psi D $. We deduce that
\begin{align*}
\timeentry \geq  \frac{2}{3 \sqrt{k}} \frac{L_{\delta,k}(\thetaInit) - \min L_{\delta,k}}{\psi D} - \frac{k K \delta^2 + 2 \delta \sqrt{k} \psi D }{\sqrt{k} \psi D }.
\end{align*}
After noting that
\[
\begin{aligned}
\bar M_{\rm SD}
&=\frac{1}{\timeentry}
  \int_0^{\timeentry}
  \sqrt{|\supp(\nabla L_{\delta,k}(\theta(t)))|}\,dt\\
&=\frac{1}{R}\int_0^R\sqrt{k}\,dt
=\sqrt{k}
\end{aligned}
\]
from \Eqref{eq:m:SD}, we know
\begin{align*}
\timeentry \geq  \frac{2}{3 \bar M_{\rm SD}} \frac{L_{\delta,k}(\thetaInit) - \min L_{\delta,k}}{\psi D} - \frac{K \bar M_{\rm SD} \delta^2 + 2 \psi \delta  D }{ \psi D }.
\end{align*}
\end{proof}

\subsubsection{Lower bound for Spectral Descent}
We consider a very similar diagonalizing process for the Spectral Descent lower bound. Let  $I_k \in \mathbb{R}^{a \times b}$ be the partial identity matrix with $1 \le k \le \min(a, b)$. Note that $\Vert{}I_k\Vert{}_F = \sqrt{k}$.
 We define the following objective function, for $k \geq 1 $:
\begin{equation}
\label{eq-def-spectral-descent-lower-bound}
\forall\Theta\in\mathbb{R}^{a\times b}, \quad L_{\delta,k}(\Theta) := \sum_{i=1}^k g_{\delta,k}(\Theta_{ii}).
\end{equation}

\begin{theorem}
\label{th:LB:min:awa:spec}
 Consider $\ThetaInit = R I_k$ with $R > 0$ and $\|\ThetaInit\|_F \leq \frac{\psi D}{K}$ and that we follow the Spectral Descent dynamics of $L_{\delta,k}$ defined in \Cref{{eq-def-spectral-descent-lower-bound}}. Then there exists an absolutely continuous solution of the spectral-descent inclusion for which the entry time is lower bounded as
\begin{align*}
\timeentry \geq  \frac{2}{3 m_{\Vert{}\cdot\Vert{}_{sp}}} \frac{L_{\delta,k}(\ThetaInit) - \min L_{\delta,k}}{\psi D} -  \frac{K m_{\Vert{}\cdot\Vert{}_{sp}} \delta^2 }{\psi D } - 2 \delta.
\end{align*}
Here $m_{\Vert{}\cdot\Vert{}_{sp}} = m_{\Vert{}\cdot\Vert{}_{sp}}(\Theta(t)) =\sqrt{k}$ ($0 \le t < \timeentry$).
\end{theorem}

\begin{proof}
First, by symmetry, there exists a trajectory that is restricted to $\Theta(t) = r(t)I_k$ with $r:\mathbb{R}_{+} \xrightarrow{} \mathbb{R}$. For $t\geq 0$, the matrix gradient is :
\begin{align*}
\nabla L_{\delta,k}(r(t) I_k) = g'_{\delta,k}(r(t)) I_k. 
\end{align*}
Similar to the Sign Descent case, we note $T_0$ the travel time for $r(t)$ from $R$ to $0$. 
Whenever $g'_{\delta,k}(r(t))>0$, we have
\[
I_k \in
\operatorname*{arg\,max}_{\|V\|_{\rm sp}\leq 1}
\operatorname{tr}\bigl((g'_{\delta,k}(r(t))I_k)^\top V\bigr)
=
\ngradnorm{\|\cdot\|_{\mathrm{sp}}}{ L_{\delta,k}(\Theta(t))}.
\]
We select the velocity $-I_k$ until $r(t)$ reaches
the first zero of $g'_{\delta,k}$, and keep the trajectory constant
thereafter. This zero lies in $[-\delta,0)$, and zero velocity is
admissible there.
Hence, we know for $t \in [0,T_0)$,
\begin{align*}
r(t) = R - t. 
\end{align*}
Therefore, for the entry time we know for $0 \leq t \leq R $
\begin{align*}
\Vert{}\nabla L_{\delta,k}(\Theta(t))\Vert{}_{nuc} = k|g_{\delta,k}'(r(t))| & = | K k  r(t) + \psi D \sqrt{k} | \\
\psi \sqrt{\text{rank}(\nabla L_{\delta,k}(\Theta(t)))} \big(\Vert{}\Theta(t)\Vert{}_F + D\big) & =\psi \sqrt{k} ( \sqrt{k}|r(t)| + D ).
\end{align*}
Therefore $\timeentry \geq T_0 = R$ as in the Sign Descent case, from \Cref{table:zoology:minority:zone}.

We now show that $\timeentry=R$. Since
$g'_{\delta,k}(0)=\psi D/\sqrt{k}>0$ and
$g''_{\delta,k}(0)=K>0$, there exists $\varepsilon\in(0,\delta)$
such that
\[
0<g'_{\delta,k}(r)<\frac{\psi D}{\sqrt{k}}
\qquad\text{for every }r\in(-\varepsilon,0).
\]
For $R<t<R+\varepsilon$, the selected trajectory
therefore still satisfies $r(t)=R-t$, its gradient has rank $k$, and
\[
\begin{aligned}
\|\nabla L_{\delta,k}(\Theta(t))\|_{\rm nuc}
&=k\,g'_{\delta,k}(r(t))\\
&<\psi D\sqrt{k}\\
&\le \psi\sqrt{k}\bigl(\|\Theta(t)\|_F+D\bigr).
\end{aligned}
\]
Thus the trajectory lies in the minority aware zone
for every $t\in(R,R+\varepsilon)$, giving $\timeentry\leq R$.
Together with $\timeentry\geq R$, this proves $\timeentry=R$.
Moreover, $m_{\Vert{}\cdot\Vert{}_{sp}}(\Theta(t))=\sqrt{k}$ for
$0\leq t<R$, and its time average satisfies
\[
\begin{aligned}
\frac{1}{\timeentry}\int_0^{\timeentry}
m_{\Vert{}\cdot\Vert{}_{sp}}(\Theta(t))\,dt
&=\frac{1}{R}\int_0^R
\sqrt{\operatorname{rank}(\nabla L_{\delta,k}(\Theta(t)))}\,dt\\
&=\sqrt{k}.
\end{aligned}
\]

We introduce the suboptimality gap and obtain a similar conclusion as for the Sign Descent case, using that $m_{\Vert{}\cdot\Vert{}_{sp}} = \sqrt{k}$ from \Eqref{eq:bar:M:spec}: 
\begin{align*}
\timeentry \geq  \frac{2}{3 m_{\Vert{}\cdot\Vert{}_{sp}}} \frac{L_{\delta,k}(\ThetaInit) - \min L_{\delta,k}}{\psi D} - \frac{Km_{\Vert{}\cdot\Vert{}_{sp}} \delta^2 + 2 \psi  \delta  D }{ \psi D }.
\end{align*}
\end{proof}

\section{Details for numerical experiments}
\label{sec-appendix-numerical-experiments}

First, we present the material and computational time required to run our experiments in \Cref{tab:model_comparison}.

\begin{table}[ht]
\begin{center}
\begin{tabular}{lccc}
\toprule
\multicolumn{1}{c}{\bf Experiment }& \multicolumn{1}{c}{\bf Dataset}  & \multicolumn{1}{c}{\bf GPU} & \multicolumn{1}{c}{\bf Computational Time} \\
\midrule
Linear Regression & Superconductor & None & <10min \\
Linear Regression & YearPredictionMSD & None & <10min \\
Classification    & Covertype & Nvidia Spark & 12H \\
Classification    & Cifar10 & Nvidia Spark & 24H \\
LLM               & WikiText-103 & Nvidia Spark & 48H \\
\bottomrule
\end{tabular}
\end{center}
\caption{Computational resources and time across different models.}
\label{tab:model_comparison}
\end{table}

\subsection{Linear regression}
\label{sec-appendix-linear-regression}
To generate a family of problems with consistent properties across varying dimensions, we apply Principal Component Analysis (PCA). We will create a series of subproblems with only $\alpha\%$ of the dimension of the original problem. 

The minority loss function is defined on a subset of $k \in \{3, \ldots ,20\}$ samples that exhibit the largest second moments in the feature matrix $X$. We then numerically compute the entry times for both optimizers alongside the mass of the minority group, $\psi$. By deriving empirical scaling laws that relate these quantities to the dimension $d$, we can directly compare the experimental results against our theoretical predictions.

The experimental protocol is defined as follows.

$\bullet$ Choose a diameter $D \in  \{0.1, 0.3, 1\} $.
    
    $\bullet$
    Choose a shared learning rate $\eta \in \left\{ \frac{1}{100\lambda_{\max}}, \frac{1}{300\lambda_{\max}}, \frac{1}{1000\lambda_{\max}} \right\}$ for both GD and SD. Here $\lambda_{\max}$ is the largest eigenvalue of the global loss Hessian.

    $\bullet$ For each configuration $(\alpha,k)$, execute both GD and SD.
    
$\bullet$ Estimate $\psi$ as the maximum eigenvalue of the Hessian of the minority $\psi \xleftarrow{} \lambda_{\max}(\nabla^2(L_0))$. 

$\bullet$ Compute $\timeentry(\mathrm{GD})$  and $\timeentry(\mathrm{SD})$  using their formal definitions.

$\bullet$ Compute $\frac{\timeentry(\mathrm{GD})}{\timeentry(\mathrm{SD})}$ for each configuration $(\alpha, k)$ and make a log log plot to fit the model $\frac{\timeentry(\mathrm{GD})}{\timeentry(\mathrm{SD})} \propto d^{c_1} \psi^{c_2}.$

We will apply the protocol to the two following dataset.

\textbf{Superconductor.} The Superconductor dataset comprises $21,000$ samples and $80$ features, with the material's critical temperature serving as the target variable. 

\textbf{YearPredictionMSD.} The YearPredictionMSD dataset contains $513,000$ samples and $90$ features, with the release year serving as the prediction target.

The results are plotted in \Cref{fig:linear-regression} and the regression results can be found in Tables \ref{tab:entry-times-linear-regression-msd} and \ref{tab:entry-times-linear-regression-superconductivity}. 

\begin{figure}
    \centering
    \includegraphics[width=0.9\linewidth]{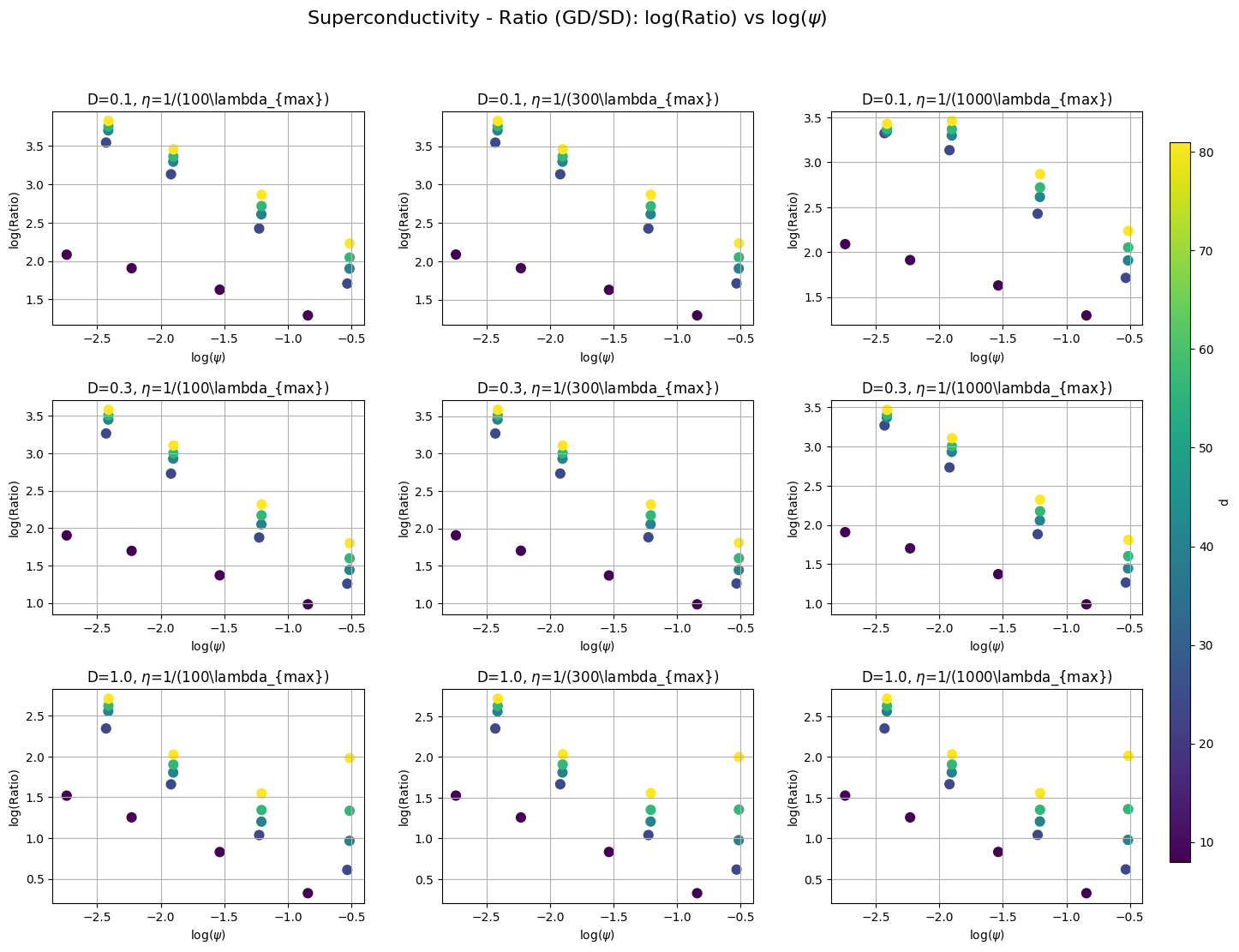}
    \includegraphics[width=0.9\linewidth]{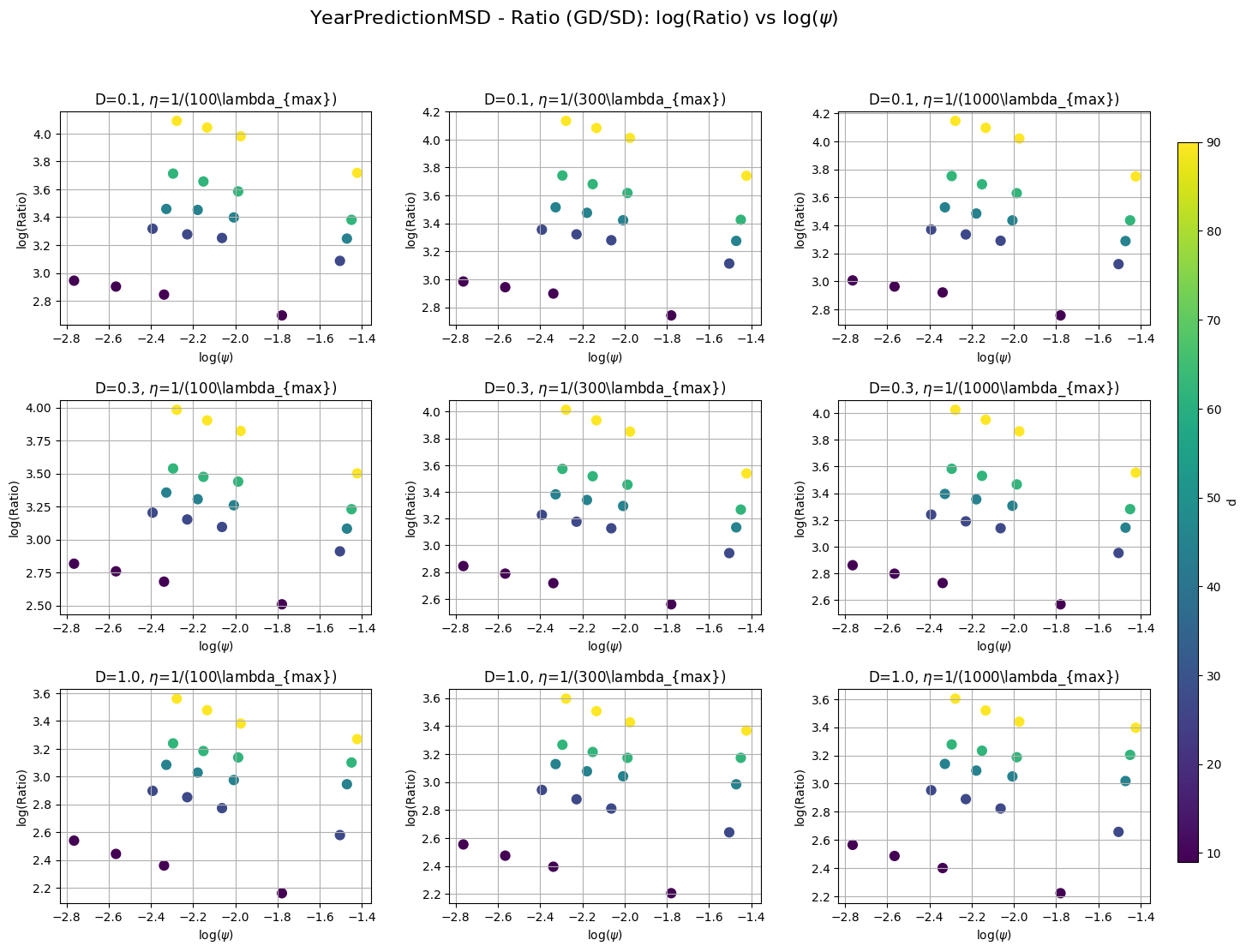}
    \caption{Ratio of Entry times for Gradient Descent and Sign Descent in the linear regression setting. The top panel corresponds to the Superconductor dataset, and the bottom panel corresponds to the YearPredictionMSD dataset.}
    \label{fig:linear-regression}
\end{figure}

The estimated log log regression coefficients are presented in \Cref{tab:entry-times-linear-regression-msd,tab:entry-times-linear-regression-superconductivity}. 

\begin{table}[ht]
\centering
\vspace{2mm}
\begin{tabular}{@{} lccc @{}}
\toprule
& $\eta = \frac{1}{100\lambda_{\max}}$ & $\eta = \frac{1}{300\lambda_{\max}}$ & $\eta = \frac{1}{1000\lambda_{\max}}$ \\
\midrule
$D = 0.1$ & $\frac{\mathcal{T}_{GD}}{\mathcal{T}_{SD}} \propto \frac{d^{0.51}}{\psi^{0.34}}$ & $\frac{\mathcal{T}_{GD}}{\mathcal{T}_{SD}} \propto \frac{d^{0.50}}{\psi^{0.35}}$ & $\frac{\mathcal{T}_{GD}}{\mathcal{T}_{SD}} \propto \frac{d^{0.50}}{\psi^{0.35}}$ \\ [3ex]
$D = 0.3$ & $\frac{\mathcal{T}_{GD}}{\mathcal{T}_{SD}} \propto \frac{d^{0.52}}{\psi^{0.39}}$ & $\frac{\mathcal{T}_{GD}}{\mathcal{T}_{SD}} \propto \frac{d^{0.51}}{\psi^{0.38}}$ & $\frac{\mathcal{T}_{GD}}{\mathcal{T}_{SD}} \propto \frac{d^{0.51}}{\psi^{0.38}}$ \\ [3ex]
$D = 1.0$ & $\frac{\mathcal{T}_{GD}}{\mathcal{T}_{SD}} \propto \frac{d^{0.49}}{\psi^{0.29}}$ & $\frac{\mathcal{T}_{GD}}{\mathcal{T}_{SD}} \propto \frac{d^{0.50}}{\psi^{0.26}}$ & $\frac{\mathcal{T}_{GD}}{\mathcal{T}_{SD}} \propto \frac{d^{0.50}}{\psi^{0.24}}$ \\ [1ex]
\bottomrule
\end{tabular}
\caption{Estimated log log regression coefficients for different choices of $(D,\eta)$ on the YearPredictionMSD dataset.}
\label{tab:entry-times-linear-regression-msd}
\end{table}

\begin{table}[ht]
\centering
\vspace{2mm}
\begin{tabular}{@{} lccc @{}}
\toprule
& $\eta = \frac{1}{100\lambda_{\max}}$ & $\eta = \frac{1}{300\lambda_{\max}}$ & $\eta = \frac{1}{1000\lambda_{\max}}$ \\
\midrule
$D = 0.1$ & $\frac{\mathcal{T}_{GD}}{\mathcal{T}_{SD}} \propto \frac{d^{0.71}}{\psi^{0.80}}$ & $\frac{\mathcal{T}_{GD}}{\mathcal{T}_{SD}} \propto \frac{d^{0.71}}{\psi^{0.80}}$ & $\frac{\mathcal{T}_{GD}}{\mathcal{T}_{SD}} \propto \frac{d^{0.64}}{\psi^{0.68}}$ \\ [3ex]
$D = 0.3$ & $\frac{\mathcal{T}_{GD}}{\mathcal{T}_{SD}} \propto \frac{d^{0.65}}{\psi^{0.91}}$ & $\frac{\mathcal{T}_{GD}}{\mathcal{T}_{SD}} \propto \frac{d^{0.65}}{\psi^{0.91}}$ & $\frac{\mathcal{T}_{GD}}{\mathcal{T}_{SD}} \propto \frac{d^{0.64}}{\psi^{0.88}}$ \\ [3ex]
$D = 1.0$ & $\frac{\mathcal{T}_{GD}}{\mathcal{T}_{SD}} \propto \frac{d^{0.55}}{\psi^{0.69}}$ & $\frac{\mathcal{T}_{GD}}{\mathcal{T}_{SD}} \propto \frac{d^{0.55}}{\psi^{0.68}}$ & $\frac{\mathcal{T}_{GD}}{\mathcal{T}_{SD}} \propto \frac{d^{0.56}}{\psi^{0.68}}$ \\ [1ex]
\bottomrule
\end{tabular}
\caption{Estimated log log regression coefficients for different choices of $(D,\eta)$ on the  Superconductivity Dataset.}
\label{tab:entry-times-linear-regression-superconductivity}
\end{table}

\subsection{Classification}
To closely align with our theoretical framework, we evaluate optimizers without momentum for the Covertype and CIFAR-10 classification tasks. Across all experiments, the batch size is fixed at $128$. Full hyperparameter configurations are detailed in \Cref{tab-hyperparameters-classification}.
\begin{table}[ht]
\begin{center}
\begin{tabular}{lc}
\toprule
Optimizers &  Parameters \\
\midrule
AdamW & $lr=10^{-3},  \beta=(0.0, 0.0),\text{weight_decay}=10^{-3} $\\
SGD & $lr=0.05,  \text{momentum}=0.0,\text{weight_decay}= 10^{-3} $ \\
Muon & $lr=0.02,  \text{momentum}=0.0, \text{adamw_lr}=10^{-3},\text{adamw_wd}=10^{-2} $\\
\bottomrule
\end{tabular}
\end{center}
\caption{Hyperparameters of optimizers for classification}
\label{tab-hyperparameters-classification}
\end{table}
We repeat all experiments across 3 fixed random seeds and report the average accuracy at each iteration. The results for CIFAR-10 are shown in \Cref{fig:cifar}. 
Those from Covertype are shown in
\Cref{tab:covertype_eval}.

\begin{figure}
    \centering
    \includegraphics[scale=0.4]{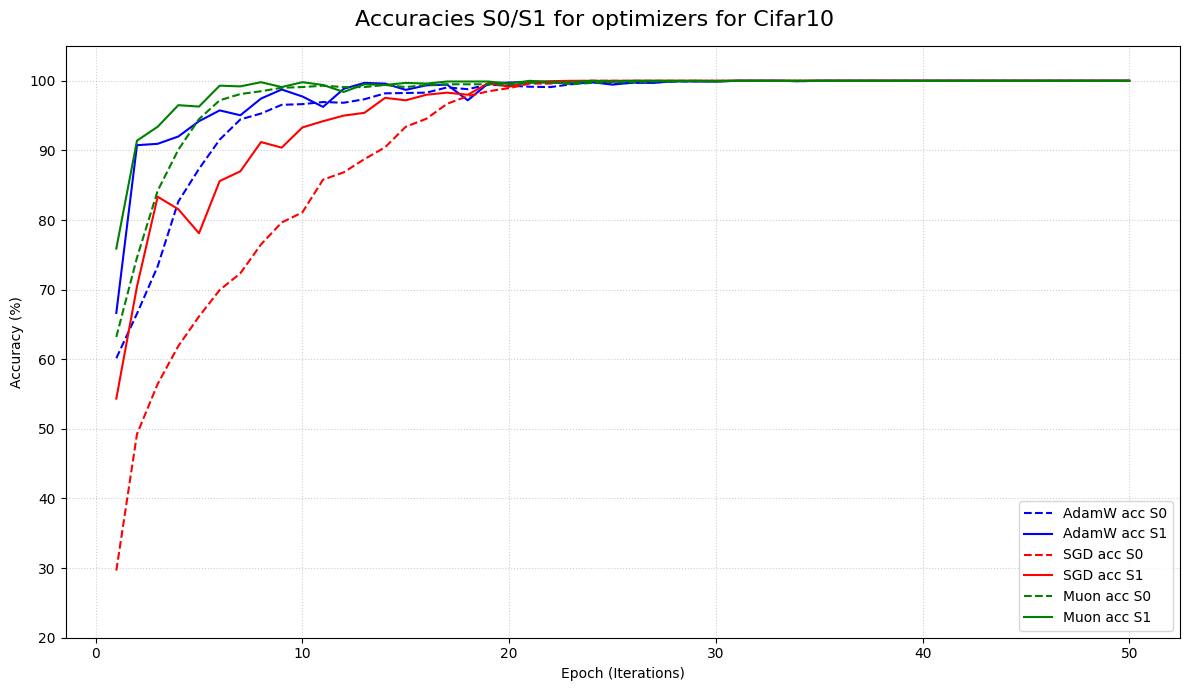}
    \caption{Evolution of $S_0/S_1$ accuracy during training phase for Cifar10. Here $S_0$ denotes the minority group and $S_1$ the majority group.}
    \label{fig:cifar}
\end{figure}

Across both the tabular and image classification experiments, we observe a consistent trend: although AdamW, SGD, and Muon achieve comparable final test performance on the majority population, their efficiency on the hidden minority diverges significantly. Notably, normalized gradient methods---such as AdamW, and particularly Muon---demonstrate substantially stronger learning capabilities on the minority group.

\begin{table}[h]
\centering
\small
\begin{tabular}{@{}lcc@{}}
\toprule
\cmidrule(lr){2-3} 
\textbf{Metric} & AdamW & SGD \\
\midrule
Average $S_0$ accuracy & $63.93\%$ & $51.40\%$ \\
Average $S_1$ accuracy & $90.99\%$ & $88.25\%$ \\
\bottomrule
\end{tabular}
\caption{\footnotesize Evaluation of minority/majority accuracy, averaged during training, on the Covertype dataset with AdamW and SGD.}
\label{tab:covertype_eval}
\end{table}

\subsection{Language modelling}
We train a SmallGPT2 model on WikiText-103 from scratch with the BPE Tokenizer. The details of the experiment can be found in \Cref{tab-llm-hyperparameters}.
We consider the normalized losses
\begin{align*}
\overline{L}_{1}(\theta) & = \frac{\sum_{i=1}^{N} L_i(\theta) \cdot \mathbb{I}_{\text{major}}(y_i)}{\sum_{i=1}^{N} \mathbb{I}_{\text{major}}(y_i)} \\
\overline{L}_{0}(\theta) & = \frac{\sum_{i=1}^{N} L_i \cdot \mathbb{I}_{\text{minor}}(y_i)}{\sum_{i=1}^{N} \mathbb{I}_{\text{minor}}(y_i)},
\end{align*}
where $\mathbb{I}_{\text{major}}(y_i)$ indicates that the token $y_i$ belongs to the majority group, and $\mathbb{I}_{\text{minor}}(y_i)$ indicates that the token $y_i$ belongs to the minority group.

To account for the distinct minimum values of these objectives, we normalize the loss trajectories using a fully pre-trained SmallGPT-2 model with parameters $\theta_\star$. By evaluating the normalized majority loss $\overline{L}_1 -\overline{L}_1(\theta_\star)$ and minority loss $\overline{L}_0 - \overline{L}_0(\theta_\star)$ , we plot the shifted objective trajectories,  to compare the optimizer dynamics.

The results are presented in \Cref{fig:smallgpt2-full}.
The optimization of the majority loss $L_1$ aligns with the established hierarchy of convergence rates, where Muon outpaces AdamW, and AdamW outperforms SGD. However, the optimization dynamics of SGD reveal a fundamental divergence: it neglects to minimize the minority loss $L_0$, prioritizing majority tokens at the direct expense of rare tokens.
Interestingly, by step 2000, the Muon-optimized parameters actually surpass the pre-trained baseline's performance on rare tokens, albeit at the cost of a degraded majority loss. This behavior suggests a potential multi-stage learning process, prompting the hypothesis that a later training phase may reverse this trade-off to recover majority performance.
\begin{table}[ht]
\begin{center}
\begin{tabular}{lc}
\toprule
 &  Settings \\
\midrule
Model architecture & $L=H=12, d_{\rm model}=768, \sigma_{\rm init} = 0.02$\\
Common Training& $\textrm{Sequence Length} =1024, \textrm{Accumulation steps = 16}, \textrm{Batch Size}=512$ \\
AdamW & $lr=5 \times 10^{-4},  \beta=(0.9, 0.999), \epsilon = 10^{-8}, \text{weight_decay}=10^{-2} $\\
SGD & $lr=0.02,  \text{momentum}=0.9 $ \\
Muon & $lr=0.02,  \text{momentum}=0.95 $\\
\bottomrule
\end{tabular}
\end{center}
\caption{Hyperparameters of optimizers for LLM}
\label{tab-llm-hyperparameters}
\end{table}

\begin{figure}[h]
    \centering
    \includegraphics[width=0.8\linewidth]{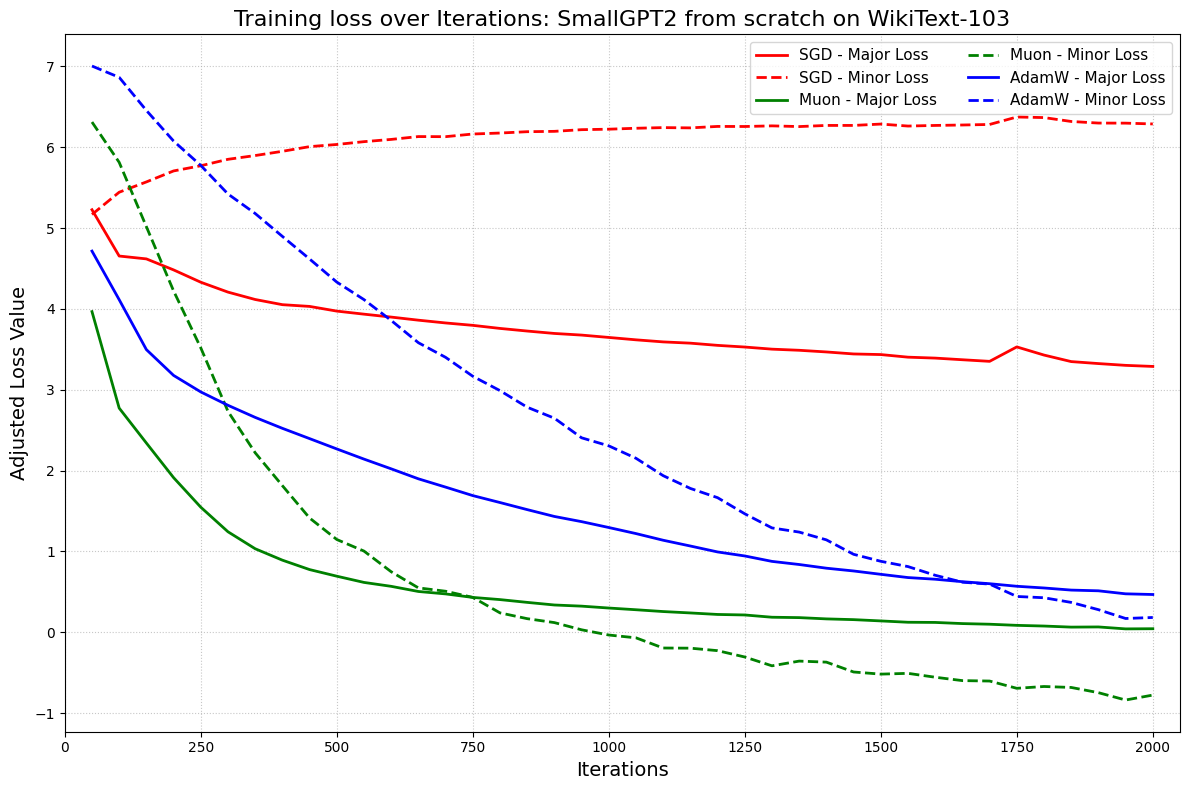}
    \caption{Evolution of training loss of SmallGPT2 from scratch on WikiText-103 over 2000 iterations. We plot the relative loss $\overline{L}_1 - \overline{L}_1(\theta_\star)$ and $\overline{L}_0 - \overline{L}_0(\theta_\star)$ for comparison.}
    \label{fig:smallgpt2-full}
\end{figure}

\section{Table of comparison of \texorpdfstring{$\kappa$}{kappa}}
\label{sec-range-of-kappa}

In Tables \ref{tab:vision}, \ref{tab:text_gen} and \ref{tab:tabular},
we now examine the values of $\kappa$ reported in the literature across three task families: image classification, autoregressive language modelling and tabular learning. Among the configurations reported below, peak learning rates range from $10^{-3}$ to $2\times10^{-2}$ for image classification, from $10^{-4}$ to $5\times10^{-2}$ for tabular learning, and from $1.5\times10^{-4}$ to $10^{-1}$ for language modelling. These ranges yield $\kappa\approx20$, $\kappa\approx500$, and $\kappa\approx700$, respectively.

From \Cref{section:intro},
recall that $\kappa=\eta_{\max}/\eta_{\min}$ denotes the ratio between the largest and smallest learning-rate magnitudes before scheduling. Applying a shared multiplicative schedule, such as cosine annealing, preserves these learning-rate ratios whenever the scheduling factor is positive.

\begin{table}[ht]
\centering
\begin{tabular}{lllllll}
\toprule
\textbf{Citation} & \textbf{Model (Size)} & \textbf{Dataset} & \textbf{Optimizer} & \textbf{Peak LR}& \textbf{Schedule} \\
\midrule
\cite{krizhevsky2012alexnet} & AlexNet (60M) & ImageNet & SGD & 0.01 &  Step \\
\cite{loshchilov2017adamw} & AlexNet (60M) & ImageNet & AdamW & 1e-3  & Cosine \\
\cite{dosovitskiy2020vit} & ViT-Base (86M) & ImageNet & AdamW & 1e-3 &  Cosine \\
\cite{ancey2026muon} & ViT-Base (86M) & ImageNet & Muon & 0.02 &  Cosine \\
\bottomrule
\end{tabular}
\caption{Image Classification, $\kappa \approx 20$.}
\label{tab:vision}
\end{table}

\begin{table}[ht]
\centering
\begin{tabular}{lllllll}
\toprule
\textbf{Citation} & \textbf{Model (Size)} & \textbf{Dataset} & \textbf{Optimizer} & \textbf{Peak LR} & \textbf{Schedule} \\
\midrule
\cite{touvron2023llama} & LLaMA (7B) & 1T tokens & AdamW & 3e-4 & Cosine  \\
\cite{touvron2023llama} & LLaMA (65B) & 1.4T tokens & AdamW & 1.5e-4  & Cosine  \\
\cite{karpathy2023nanogpt} & NanoGPT (124M) & OpenWebText & AdamW & 6e-4  & Cosine \\
\cite{jin2026modded} & NanoGPT (124M) & OpenWebText & SGD (Momentum) & $\sim$0.1 & Cosine \\
\cite{jordan2024muon} & NanoGPT (124M) & OpenWebText & Muon & 0.02  & Varies \\
\cite{jin2026modded} & Modded-NanoGPT & Fineweb & Muon & 0.025 & Trapezoidal \\
\bottomrule
\end{tabular}
\caption{Autoregressive Language Modeling, $\kappa \approx 700$.}
\label{tab:text_gen}
\end{table}

\begin{table}[ht]
\centering
\begin{tabular}{lllllll}
\toprule
\textbf{Citation} & \textbf{Model / Architecture} & \textbf{Dataset} & \textbf{Optimizer} & \textbf{Peak LR}  & \textbf{Schedule} \\
\midrule
\cite{gorishniy2021revisiting} & ResNet (Tabular) & Adult, Higgs, CovType & AdamW & 1e-3 & Constant / Step \\
\cite{arik2021tabnet} & TabNet & Forest Cover, KDD & Adam & 1e-3  & Step Decay \\
\cite{arik2021tabnet} & TabNet & Rossmann, Adult & Adam & 2e-2 &  Step Decay \\
\cite{wang2024tabkanet} & TabKANet & General Tabular Reg. & SGD & 1e-4  & Step \\
\cite{qu2026tabiclv2betterfasterscalable} & TabICLv2 & Synthetic & Muon & 1e-4 - 2e-5 &  Cosine / Trap \\
\bottomrule
\end{tabular}
\caption{Tabular Data, $\kappa \approx 1000$.}
\label{tab:tabular}
\end{table}

\section{Further Geometric Considerations under strong convexity}\label{derdeder}
In this section, we examine how strong convexity of the loss function impacts our theoretical results.

\begin{lemma}
Suppose that $L$ is $\mu$-strongly-convex, $L_0 \in \Kmin$ and $\hat{\theta} = \arg\min L$. Then  
\[ \forall \theta \in \mathbb{R}^d, \quad |L_0(\theta) - L_0(\hat{\theta})| \leq \frac{\psi}{2 \mu^2} \|  \nabla L(\theta) \|_2^2  + \frac{\psi}{\mu} \| \nabla L(\theta) \|_2 (  \| \hat{\theta}\|_2 + D ).\]
\end{lemma}
\begin{proof}
Fix $\theta \in \mathbb{R}^d$ and consider $g(t) := L_0((1-t)\hat{\theta} + t \theta)$. Since $L_0$ is $C^2$, we deduce that
\begin{align*}
g(1) - g(0) & = \int_{0}^{1} g'(s) ds \\
|L_0(\theta) - L_0(\hat{\theta})| & = \Big| \int_{0}^{1} \langle \nabla L_0((1-s)\hat{\theta} + s \theta), \theta - \hat{\theta} \rangle ds \Big| \\
& \leq \int_{0}^{1} \| \nabla L_0((1-s)\hat{\theta} + s \theta) \|_2 \| \theta - \hat{\theta} \|_2 ds \\
& \leq \psi \int_{0}^{1} \Big( \|((1-s)\hat{\theta} + s \theta)\|_2 + D \Big) \| \theta - \hat{\theta} \|_2 ds \\
& \leq \psi \int_{0}^{1} \Big(  s\|\hat{\theta} - \theta\|_2 + \| \hat{\theta} \|_2+ D \Big) \| \theta - \hat{\theta} \|_2 ds. 
\end{align*}
We deduce that
\begin{align*}
|L_0(\theta) - L_0(\hat{\theta})| & \leq \psi \| \theta - \hat{\theta} \|_2  \Big( \frac{\| \theta - \hat{\theta} \|_2 }{2} + \| \hat{\theta}\|_2 + D \Big) \\
& \leq \frac{\psi}{2 \mu^2} \| \nabla L(\theta) \|_2^2  + \frac{\psi}{\mu} \| \nabla L(\theta) \|_2 (  \| \hat{\theta}\|_2 + D )
\end{align*}
since $\| \nabla L(\theta) \|_2 \geq \mu \|\theta - \hat{\theta} \|_2$ due to the strong convexity of $L$.
\end{proof}
Building on the preceding lemma, strong convexity allows us to directly bound the minority loss gap, $|L_0(\theta) - L_0(\hat{\theta})| < \epsilon$ in terms of $\psi$ and the gradient norm  $\| \nabla L\|_2$.

\begin{lemma}
Under \Cref{ass-cost,ass-VL}, suppose that $L$ is $\mu$-strongly-convex. Then for any fixed $L$ and $\thetaInit$ such that $\nabla L(\thetaInit) \neq 0$,
\begin{align*}
\forall \psi,D > 0, \quad  & \timeentry^{\mathrm{SD}}  \leq \sqrt{\frac{2(L(\thetaInit) - L_\star)}{\mu}} \\
\forall D>0, \quad & \timeentry^{\mathrm{GD}} \xrightarrow[\psi \xrightarrow{} 0]{} + \infty.
\end{align*}
\end{lemma}
\begin{proof}
We note that by strong-convexity :
\[ \forall \theta, \quad \frac{1}{2}\| \nabla L(\theta)\|_2^2 \geq \mu \Big( L(\theta) - L(\hat{\theta}) \Big).\]
\textbf{For Sign Descent dynamism}, by the chain rule, we obtain for  $t \geq 0$
\begin{align*}
\frac{d}{dt} \Big( L(\theta(t)) - L_\star \Big) & = - \| \nabla L(\theta(t))\|_1 \leq - \| \nabla L(\theta(t))\|_2 \\
\implies \frac{d}{dt} \Big(\sqrt{ L(\theta(t)) - L_\star }\Big) & = \frac{- \| \nabla L(\theta(t))\|_1}{2 \sqrt{ L(\theta(t)) - L_\star}} \leq -\frac{1}{2} \sqrt{2\mu} \\
\implies - \sqrt{ L(\thetaInit) - L_\star } &  \leq -\sqrt{\frac{\mu}{2}} \timeentry^{\mathrm{SD}} \\
\implies \timeentry^{\mathrm{SD}} & \leq \sqrt{\frac{2(L(\thetaInit) - L_\star)}{\mu}}.
\end{align*}

\textbf{For Gradient Descent}, we start by defining
\begin{align*}
\mathcal{S}(\thetaInit) := \{ \theta \in \mathbb{R}^d, L(\theta) \leq L(\thetaInit) \}.
\end{align*}
By the strong convexity of $L$, we know that for all $\theta$
\begin{align*}
L(\theta) - L(\thetaInit) \geq \langle \nabla L(\thetaInit), \theta - \thetaInit \rangle + \frac{\mu}{2} \| \theta - \thetaInit\|_2^2. 
\end{align*}
Therefore
\begin{align*}
L(\theta) \leq L(\thetaInit) \implies \| \nabla L(\thetaInit)\|_2 \| \theta - \thetaInit\|_2 \geq \frac{\mu}{2} \| \theta - \thetaInit\|_2^2.
\end{align*}
If $\theta \neq \thetaInit$, we have $ \| \nabla L(\thetaInit)\|_2  \geq \frac{\mu}{2} \| \theta - \thetaInit \|_2 $. Therefore 
\[ \theta \in \mathcal{S}(\thetaInit) \implies \| \theta - \thetaInit\|_2 \leq  \frac{2}{\mu} \| \nabla L(\thetaInit)\|_2. \]
Then $\mathcal{S}(\thetaInit)$ is compact, we note 
\[ \lambda_{\max} := \max_{\theta \in \mathcal{S}(\thetaInit)} \sigma_{\max}(\nabla^2 L(\theta)).\]

By the chain rule, we have for $t\geq 0$
\begin{align*}
\frac{d}{dt} \| \nabla L(\theta(t))\|_2^2 & = - 2 \langle \nabla L(\theta(t)), \nabla^2 L(\theta(t)) \nabla L(\theta(t)) \rangle \\
&  \geq - 2 \lambda_{\max}  \| \nabla L(\theta(t))\|_2^2   \\ 
\implies \| \nabla L(\theta(t))\|_2^2 & \geq \| \nabla L(\thetaInit) \|_2^2 e^{-2\lambda_{\max}t }.
\end{align*}
Since, from \Cref{table:zoology:minority:zone}, 
\[ 
\|\nabla L(\theta(\timeentry^{\mathrm{GD}}))\|_2 \le \psi (\| \theta(\timeentry^{\mathrm{GD}}) \|_2 + D) \leq \psi (\| \thetaInit \|_2 + \frac{2}{\mu} \| \nabla L(\thetaInit)\|_2 + D),
\]
we obtain
\begin{align*}
\timeentry^{\mathrm{GD}} & \geq \frac{1}{\lambda_{\max}}\log \left( \frac{\| \nabla L(\thetaInit)\|_2}{\|\nabla L(\theta(\timeentry^{\mathrm{GD}}))\|_2}\right) \\
& \geq \frac{1}{\lambda_{\max}}\log \left( \frac{\| \nabla L(\thetaInit)\|_2}{\psi (\| \thetaInit \|_2 + \frac{2}{\mu} \| \nabla L(\thetaInit)\|_2 + D)}\right) \xrightarrow[\psi \xrightarrow{} 0]{} + \infty.
\end{align*}
\end{proof}
The previous lemma demonstrates that for strongly convex objective functions, normalized methods such as Sign Descent exhibit fundamentally different entry time dynamics compared to standard Gradient Descent.

\end{document}

%% file: math_commands.tex
\usepackage{amsmath,amsfonts,bm}

\def\eqref#1{equation~\ref{#1}}
\def\Eqref#1{Equation~\ref{#1}}

\def\1{\bm{1}}

\DeclareMathAlphabet{\mathsfit}{\encodingdefault}{\sfdefault}{m}{sl}
\SetMathAlphabet{\mathsfit}{bold}{\encodingdefault}{\sfdefault}{bx}{n}

\newcommand{\R}{\mathbb{R}}

\DeclareMathOperator*{\argmax}{arg\,max}

\DeclareMathOperator{\sign}{sign}